\documentclass[twoside,11pt]{article}

\usepackage{jmlr2e}

\usepackage{hyperref}
\usepackage{url}
\usepackage{amsfonts}
\usepackage{amsmath}
\usepackage{amssymb}
\usepackage{bm,amsbsy} 
\usepackage{algorithm,algorithmic}
\usepackage{wrapfig} 
\usepackage{sidecap} 
\usepackage{pgf,tikz,pgfplots}
\usetikzlibrary{arrows,automata,fit,positioning}

\newcommand{\Nrm}{\mathcal{N}}

\newcommand{\vx}{\mathbf{x}}

\newcommand{\vw}{\mathbf{w}}
\newcommand{\vm}{\mathbf{m}}

\newcommand{\vy}{\mathbf{y}}
\newcommand{\trp}{{^\top}} 

\newcommand{\inv}{^{-1}}

\newcommand{\vmu}{\mathbf{\ensuremath{\bm{\mu}}}}
\newcommand{\vtheta}{\mathbf{\ensuremath{\bm{\theta}}}}
\newcommand{\veta}{\mathbf{\ensuremath{\bm{\eta}}}}
\newcommand{\vv}{\mathbf{v}}
\newcommand{\vg}{\mathbf{g}}
\newcommand{\vh}{\mathbf{h}}
\newcommand{\vu}{\mathbf{u}}
\newcommand{\vb}{\mathbf{b}}
\newcommand{\vz}{\mathbf{z}}
\newcommand{\vones}{\mathbf{1}}

\newcommand{\nsevar}{\sigma^2}

\newcommand{\RR}{\mathbb{R}}

\newcommand{\tonehlf}{\tfrac{1}{2}}

\newcommand{\figref}[1]{Fig.~\ref{fig:#1}}
\newcommand{\diag}{\mbox{diag}}
\newtheorem{thm}{Theorem}
\usepackage{color}
\definecolor{darkred}{rgb}{0.6, 0, 0}
\definecolor{gray}{RGB}{0.5,0.5,0.5}

\usepackage{algorithmic}

\ShortHeadings{Dependent relevance determination for smooth and structured sparse regression}{Wu, Koyejo and Pillow}
\firstpageno{1}

\begin{document}

\title{Dependent relevance determination for smooth and structured sparse regression}

\author{\name Anqi Wu \email anqiw@princeton.edu \\
 \addr Princeton Neuroscience Institute\\
 Princeton University\\
 Princeton, NJ 08544, USA
 \AND
 \name Oluwasanmi Koyejo \email sanmi@illinois.edu\\
 \addr Beckman Institute for Advanced Science and Technology \\
 	Department of Computer Science\\
	University of Illinois at Urbana-Champaign\\
	Urbana, Illinois, 61801, USA
 \AND
 \name Jonathan Pillow \email pillow@princeton.edu \\
 \addr Princeton Neuroscience Institute\\
 Princeton University\\
 Princeton, NJ 08544, USA}

\editor{}

\maketitle

\begin{abstract}
In many problem settings, parameter vectors are not merely sparse but dependent in such a way that non-zero coefficients tend to
 cluster together. We refer to this form of dependency as ``region
 sparsity.'' Classical sparse regression methods, such as the lasso
 and automatic relevance determination (ARD), which model parameters as
 independent {\it a priori}, and therefore do not exploit such
 dependencies. Here we introduce a hierarchical model for smooth,
 region-sparse weight vectors and tensors in a linear regression
 setting. Our approach represents a hierarchical extension of the
 relevance determination framework, where we add a transformed
 Gaussian process to model the dependencies between the prior
 variances of regression weights. We combine this with a structured
 model of the prior variances of Fourier coefficients, which
 eliminates unnecessary high frequencies. The resulting prior
 encourages weights to be region-sparse in two different bases
 simultaneously. We develop Laplace approximation and Monte Carlo Markov Chain (MCMC) sampling to provide efficient inference for the posterior. Furthermore, a two-stage convex relaxation of the Laplace approximation approach is also provided to relax the inevitable non-convexity during the optimization. We finally show substantial improvements over comparable methods for both simulated and real datasets from brain imaging.
\end{abstract}


\section{Introduction}
Recent work in statistics has focused on high-dimensional inference
problems in which the number of parameters equals or exceeds the
number of samples. We focus specifically on the linear regression setting: 
consider a scalar response $y_i \in \RR$ generated from an input vector
$\vx_i \in {\RR}^p$ via the linear model:
\begin{equation}
y_i = \vx_i\trp \vw + \epsilon_i , \quad \text{for} \quad i = 1, 2, \cdots, n,
\end{equation} with observation noise $\epsilon_i \sim \Nrm(0,
\nsevar)$. The regression (linear weight) vector $\vw \in \RR^p$ is
the quantity of interest. This general problem is ill-posed when $n\leq p$. However, it is surprisingly tractable when $\vw$ has special structure, such as sparsity in an appropriate basis. A large literature has provided theoretical
guarantees about the solvability of such problems, as well as a suite
of practical methods for solving them.


Methods based on simple sparsity such as the lasso
\citep{tibshirani1996regression} typically treat regression weights as
independent {\it a priori}. This neglects a statistical feature of many
real-world problems, which is that non-zero weights tend to arise in local
groups or clusters. In many problem settings, weights have an explicit geometric
relationship, such as indexing in time (e.g., time series regression) or space
(e.g., brain imaging data). If a single regression weight is non-zero, nearby
weights in time or space are also likely to be non-zero. Conversely, in a region
where most weights are zero, any particular coefficient is also likely to be
zero. Thus, nearby weights exhibit dependencies that are not captured by
independent priors. We refer to this form of dependency as {\it region
 sparsity}.


A variety of methods have been developed to incorporate local dependencies
between regression weights, such as the group lasso \citep{yuan2006model}.
However, these methods typically require the user to pre-specify the group size
or to partition the weights into groups {\it a priori}. Such information is
unavailable in many applications of interest, and hard partitioning into groups
breaks dependencies between nearby coefficients that are assigned to different
groups.

In this paper, we take a Bayesian approach to inferring regression weights with
region-sparse structure. We introduce a hierarchical prior over $\vw$ of the
form: \begin{align}
 \vu &\sim \mathcal{GP}\\
 \vw | \vu &\sim \Nrm(0, C(\vu)),
\end{align}
where $\vu$ is a latent vector that captures dependencies in the sparsity
pattern of $\vw$, and $\vw | \vu$ has a zero-mean Gaussian distribution with a
diagonal covariance matrix $C(\vu)$, given by a deterministic function of $\vu$.
We use a Gaussian process (GP) prior over $\vu$ to encode structural assumptions
about region sparsity (e.g., the typical size of clusters of non-zero weights
and the spacing between them). This model can be seen as an extension of
automatic relevance determination (ARD), in which the elements of $\vu$ are {\it
 a priori} independent \citep{mackay1992Bayesian, neal1995Bayesian}. We therefore refer to it as {\it
 dependent relevance determination} (DRD).


Note that region-sparsity refers only to the sparsity pattern of regression
weights, i.e., the locations where they are non-zero, not to the particular
values of the weights themselves. This is reflected in the fact that we define
the DRD prior covariance matrix $C(\vu)$ to be diagonal, making the weights
conditionally independent given the pattern defined by $\vu$. In many cases,
however, we expect weights to be smooth as well as sparse due to the continuity of
the input regressors in space or time. Most of the real datasets do exhibit spatial and temporal correlations. 
Coefficients usually possess contiguous regions and smoothness.
Hence, we are aiming at developing a universal approach easily integrating both structured sparsity and smoothness concurrently.
To incorporate smoothness, we combine the
standard DRD prior with a squared exponential covariance function. The resulting
prior has a non-diagonal covariance matrix that encourages smoothness as well as
sparsity. We refer to this extension as {\it smooth dependent relevance determination} (smooth-DRD). Samples from the smooth-DRD prior
have local islands of smooth and non-zero weights, surrounded by large regions of
zeros. We will show that combining region-sparsity and smoothness together will significantly enhance the performance in a non-trivial way.

Unfortunately, exact inference under DRD and smooth-DRD priors is analytically
intractable. We therefore introduce an approximate inference method based on a Laplace approximation to the posterior over $\vu$, and a sampling-based
inference method using Monte Carlo Markov Chain (MCMC) sampling. We also derive
a two-stage convex relaxation of the Laplace approximation approach in order to
overcome the effects of bad local optima.

We show experimental evaluations on 1D simulated datasets comparing
the performance among different methods. In addition, the phase
transition curve analyses are carried out against lasso to show the
superiority of DRD and smooth-DRD in support recovery for group
structure sparsity with or without smoothness. Furthermore, the DRD
based priors are exploited for three brain imaging datasets. Domain
expertise and current evidence in brain imaging suggest that
discrimination performance is primarily driven by spatially smooth
activation within spatially sparse regions, and several estimation
algorithms have been proposed that exploit this structure
\citep{grosenick2011family,michel2011total,baldassarre2012structured,gramfort2013identifying}. We
provide experimental comparisons to these methods, showing the
superiority of DRD in practice. In particular, DRD provides spatial
decoding weights for brain imaging data that are both more
interpretable and achieve higher decoding performance.


%

Here we highlight our key contributions as follows:
\begin{itemize}
\item We introduce a new hierarchical model for smooth, region-sparse weight tensors. The model uses a Gaussian process to introduce dependencies between prior variances of
regression weights governing localized sparsity in weights and simultaneously imposes smoothness by integrating a smoothness-inducing covariance function into the prior distribution of weights.
\item We describe two methods for inferring the model parameters: one based on the Laplace approximation and a second based on MCMC. We propose a fast approximate inference method based on the Laplace approximation involving a novel two-stage convex relaxation of the log posterior in order to overcome the effects of bad local optima.
\item We show phase transition curves governing the transition from imperfect to near-perfect recovery for lasso and DRD estimators, revealing that group structure and smoothness can have a major impact on the recoverability of sparse signals.
\end{itemize}

This paper is organized as follows. In Sec.\ 2, we review the related structured
sparsity literature. In Sec.\ 3, we introduce our new region-sparsity and
smoothness inducing priors. In Sec.\ 4, we propose two approaches to Bayesian inference for parameter estimation, the evidence optimization via
Laplace approximation and the MCMC sampling. A two-stage
convex relaxation of the Laplace approximation approach is also introduced to alleviate the non-convexity with a more robust two-stage convexity. Sec.\ 5
introduces a detailed analysis of the structured sparsity and smoothness properties
of the DRD based priors and the other methods that can be used for this purpose.
Sec.\ 6 presents the phase transition analysis for lasso and DRD estimators. Sec.\ 7 shows some experiments on three real brain imaging datasets, comparing different methods that can be used for structured sparsity. Finally, Sec.\ 8
presents the conclusion and discussion of this work.

\section{Related work}

The classic method for sparse variable selection is the lasso,
introduced by \cite{tibshirani1996regression}, which places an $l_1$
penalty on the regression weights. This method can be interpreted as
a {\it maximum a posteriori} (MAP) estimate under a Laplace (or
double-exponential) prior. A fully Bayesian treatment of this model
was later developed by \cite{park2008Bayesian}. A variety of
Bayesian methods based on other sparsity-inducing prior distributions
have been developed, including the horseshoe prior
\citep{carvalho2009handling}, which uses a continuous density with an
infinitely tall spike at the origin and heavy tails, and the
spike-and-slab prior \citep{mitchell1988Bayesian} which consists of a
weighted mixture of a delta function (the spike) and a
broad Gaussian (the slab), both centered at the origin. 

Another approach to sparse variable selection comes from empirical
Bayes (also known as evidence optimization or ``type-II'' marginal
likelihood). These methods rely on a two-step inference procedure:
(1) optimize hyperparameters governing the sparsity pattern via ascent
of the marginal likelihood; and then (2) compute MAP estimates of the
parameters given the hyperparameters. The most popular such estimator
is automatic relevance determination (ARD), which prunes unnecessary
coefficients by optimizing the precision of each regression coefficient
under a Gaussian model \citep{mackay1992Bayesian, neal1995Bayesian}.
The relevance vector machine (RVM) was later formulated as a general
Bayesian framework for obtaining sparse solutions to regression and
classification tasks \citep{tipping2001sparse}. The RVM has an
identical functional form to the support vector machine, but provides
probabilistic analysis. \cite{tipping2003fast} then was proposed for
RVM to scale up to large scale training procedure.

%


All these methods can be interpreted as imposing a sparse
and independent prior on the regression weights. The resulting posterior
over weights has high concentration near the axes, so that many
weights end up at zero unless forced away strongly by the likelihood. 

In the field of structured sparsity learning, group lasso is the
most straightforward extension of lasso to capture sparsity existing
across collections of variables \citep{yuan2006model}. They achieved
the group sparse structure by introducing an $l_1$ penalty on the
$l_2$ norms of each group. Moreover, \cite{huang2011learning}
generalized the group sparsity idea by using coding complexity
regularization methods associated with the structure. A variety of
other papers have proposed alternative approaches to correlated or
structured regularization \citep{jacob2009group, liu2009nonparametric,
 kim2009statistical, friedman2010note, jenatton2011structured, kowalski2013social}.

Previous literature has also explored Bayesian methods for structured
sparse inference. A common strategy is to introduce a latent
multivariate Gaussian that controls the correlation structure
governing conditionally independent densities over
coefficients. \cite{gerven2009Bayesian} extended the univariate
Laplace prior to a novel multivariate Laplace distribution represented
as a scale mixture that induces coupling. \cite{hernandez2013learning}
described a similar approach that results in a marginally horseshoe
prior. Several other papers have proposed dependent generalizations
of the spike-and-slab prior. \cite{hernandez2013generalized}
described a group spike-and-slab distribution using a multivariate
Bernoulli distribution over the indicators of the spikes associated
with a group specification. Subsequently, \cite{andersen2014Bayesian,Andersen15} relaxed the hard-coded group
specification by encoding the structure with a generic covariance
function. Meanwhile, \cite{engelhardt2014Bayesian} introduced a
Bayesian model for structured sparsity that uses a Gaussian process
(GP) to control the mixing weights of the spike and slab prior in
proportion to feature similarity. Apart from imposing the correlation
structure on the independent spike and slab elements,
\cite{yu2012Bayesian} put forward a hierarchical Bayesian framework
with the mixing weights of the cluster patterns generated from Beta
distributions. Our work is most similar to \cite{engelhardt2014Bayesian} and \cite{Andersen15}, except that we use an
ARD-like approach with a conditionally Gaussian density over
coefficients instead of a spike and slab prior. Our work is also the
first that we are aware of that simultaneously captures sparsity and
smoothness.

\section{Dependent relevance determination (DRD) priors}
In this section, we introduce the DRD prior and the smooth-DRD prior, an
extension to incorporate smoothness of regression weights. We focus
on the linear regression setting with conditional responses
distributed as:
\begin{equation}\label{ll}
\vy|X, \vw, \nsevar \sim \Nrm(\vy|X \vw, \nsevar I),
\end{equation} 
where $X = [\vx_1,\ldots, \vx_n]\trp \in \RR^{n
 \times p}$ denotes the design matrix, $\vy = [y_1, \cdots,
y_n]\trp\in \RR^{n}$ is the observation vector, and $\sigma^2$ is the
observation noise variance, where $p$ is the dimension of the
input vectors and $n$ is the number of samples. 

 \begin{figure}[!t] 
 \centering
 \includegraphics[width=0.9\textwidth]{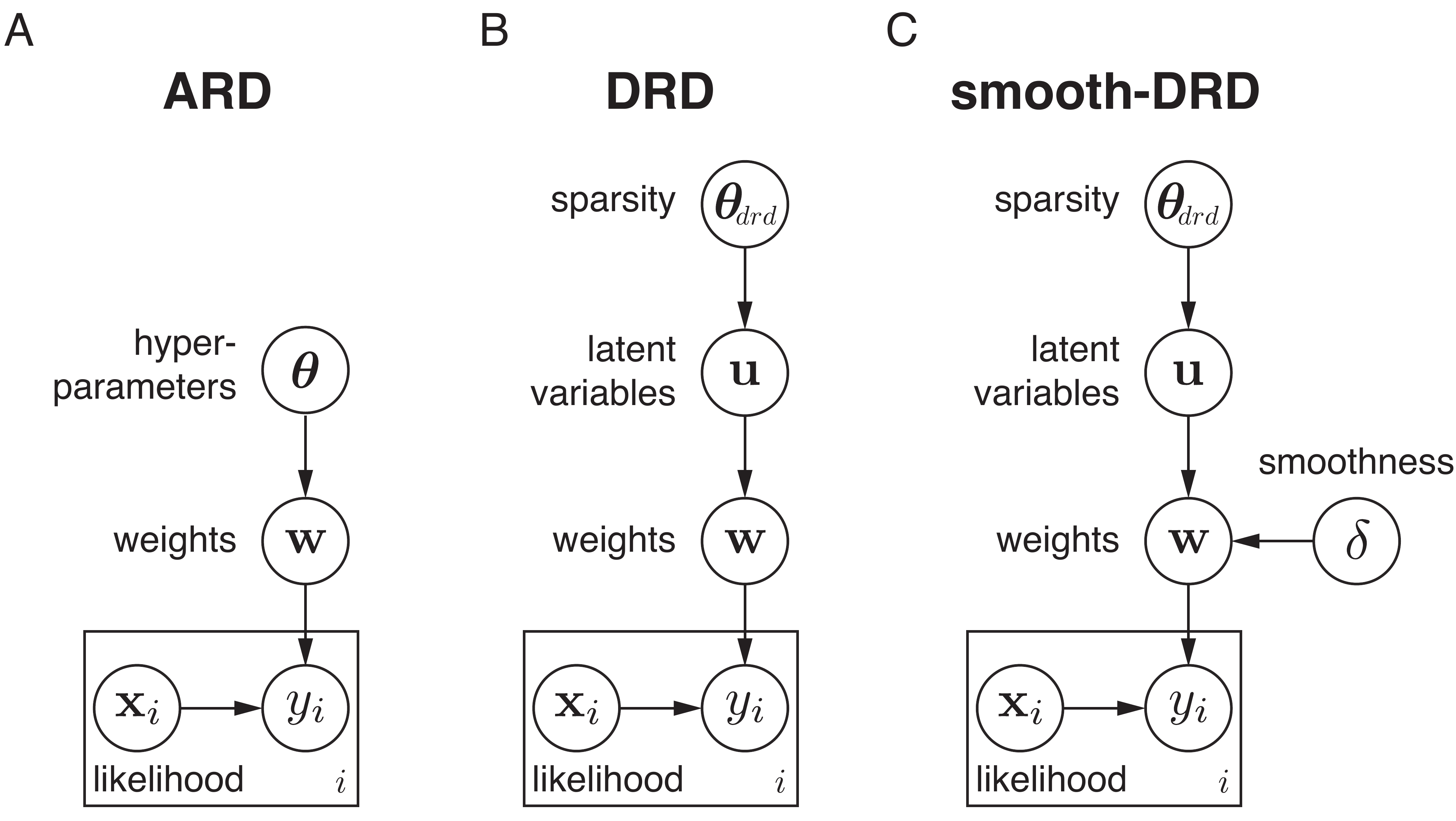} 
\caption{Graphical models for ARD, DRD and smooth-DRD.} 
\label{fig:graphical_model}
 \end{figure}

\subsection{Automatic relevance determination}

The relevance determination framework includes a family of estimators
that rely on a zero-mean multivariate normal prior:
\begin{equation}\label{wprior}
\vw |\vtheta \sim \Nrm(0, C(\vtheta)),
\end{equation} where the prior covariance matrix $C(\vtheta)$ is a
function of some hyperparameters $\vtheta$. The form of the dependence
of $C$ on $\theta$ leads to different forms of assumed structure,
including sparsity
\citep{tipping2001sparse, tipping2002analysis, Wipf08}, smoothness
\citep{sahani2003evidence, schmolck2008smooth}, or locality
\citep{park2011receptive}.

Automatic relevance determination (ARD) defines the prior covariance
to be diagonal, $C_{ii}=\theta_i^{-1}$, where a distinct
hyperparameter $\theta_i$ specifies the prior precision for the $i$'th
regression coefficient. ARD places an independent improper gamma
prior on each hyperparameter, $\theta_i \sim \textrm{gamma}(0,0)$, and
performs inference for $\{\theta_i\}$ by maximizing the marginal
likelihood. This sends many $\theta_i$ to infinity, pruning the
corresponding coefficients out of the model. A typical graphical model for ARD is presented in \figref{graphical_model}A.
The independence
assumption in the prior over hyperparameters means that there is no
tendency for nearby coefficients to remain in or be pruned from the
model. This is the primary shortcoming that our method seeks to
overcome.



\subsection{DRD: A hierarchical extension of ARD}

We extend the standard ARD model by adding a level of hierarchy.
Instead of directly optimizing hyperparameters that control sparsity
of each weight, as in ARD, we introduce a latent vector governed by a
GP prior to capture dependencies in the sparsity pattern over weights
(see \figref{graphical_model}B). Let $\vu \in \RR^{p}$ denote a latent
vector distributed according to a GP prior
\begin{align}
 \vu &\sim \mathcal{GP}(b\vones,K), \label{exp_non} 
 \end{align}
 where $b \in \RR$ is the scalar mean, $\vones$ is a length-$p$ vector of
 ones, and covariance matrix $K$ is determined by a squared
 exponential kernel. The $i,j$'th entry of $K$ is given by
\begin{equation} 
K_{ij} = \rho \exp\left( - \frac{||\chi_i-\chi_j||^2}{2l^2}\right), \label{eq:K}
\end{equation}
where $\chi_i$ and $\chi_j$ are the spatial locations of weights $w_i$
and $w_j$, respectively, and kernel hyperparameters are the marginal
variance $\rho>0$ and length scale $l>0$. Samples from this GP on a
grid of locations $\{\chi_i\}$ are smooth on the scale of $l$, and
have mean $b$ and marginal variance $\rho$.

To obtain a prior over region-sparse weight vectors, we transform
$\vu$ to the positive values via a nonlinear function $f$, and the
transformed latent vector $\vg = f(\vu)$ forms the diagonal of a
diagonal covariance matrix for a zero-mean Gaussian prior over the
weights:
\begin{align}
 C_{drd} &= {\mbox{diag}}\Big[f(\vu)\Big],
\end{align}
where $f$ is a monotonically increasing function that squashes
negative values of $\vu$ to near zero. Here we will mainly consider the exponential function $f(u) = \exp(u)$, but we will also consider ``soft-rectification'' function $f(u) = \log(1+\exp(u))$ in the experiment for numerical stability. When the GP
mean $b$ is very negative relative to the prior standard deviation
$\sqrt{\rho}$, most elements of $\vg$ will be close to zero, resulting
in weights $\vw$ with a high degree of sparsity (i.e., few weights far
from zero). The length scale $l$ determines the smoothness of samples
$\vu$ and thereby determines the typical width of bumps in the prior
variance $\vg$. We denote the set of hyperparameters governing the GP
prior on $\vu$ by $\vtheta_{drd} = \{ b, \rho, l \}$. \figref{drdgm}A
shows a depiction of sampling from the DRD generative model.
\begin{figure}[!t] 
 \centering
 \includegraphics[width=1\textwidth]{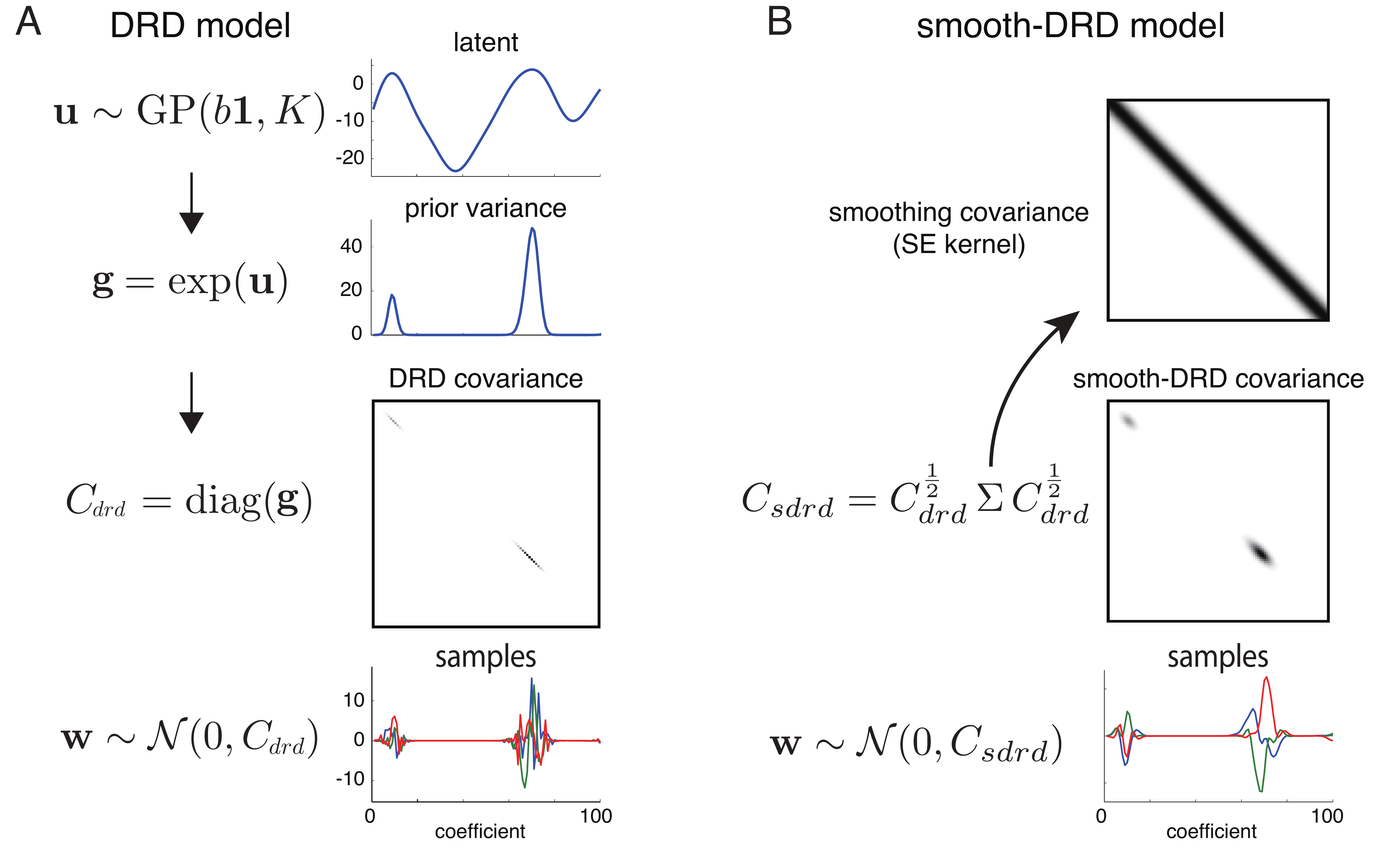}
 \caption{The sampling procedures for the generative models of DRD and smooth-DRD.}
 \label{fig:drdgm}
\end{figure}

\subsection{Smooth-DRD}
The standard DRD model imposes smooth dependencies in the prior
variances of the regression weights, but the weights themselves remain
uncorrelated (as reflected by the fact that the covariance $C_{drd}$
is diagonal). In many settings, however, we expect weights to exhibit
smoothness in addition to region sparsity. To capture this property,
we can augment DRD with a second Gaussian process, denoted as smooth-DRD, that induces
smoothness, contributing off-diagonal structure to the prior
covariance matrix while preserving the marginal variance pattern
imposed by DRD (see Fig. 1C).

Let $\Sigma$ denote a covariance matrix governed by a standard
squared-exponential GP kernel:
\begin{equation}
\Sigma_{ij} = \exp\left( - \frac{||\chi_i-\chi_j||^2}{2\delta^2}\right),
\end{equation}
with length scale $\delta$ and marginal variance set to 1.
Then we define the {\it smooth-DRD} covariance as the ``sandwich''
matrix given by:
\begin{equation}\label{csf}
C_{smooth-DRD} = C_{drd}^{\frac{1}{2}}\, \Sigma\, C_{drd}^{\frac{1}{2}},
\end{equation}
where $C_{drd}^{\frac{1}{2}}$ is simply the matrix square root of the
diagonal covariance matrix $C_{drd}$. The resulting matrix has the
same diagonal entries as $C_{drd}$, but has off-diagonal structure
governed by $\Sigma$ that induces smoothness. This matrix is positive semi-definite because, for all $\vx\in\RR^p$, $\vx\trp C_{smooth-DRD} \vx = (C_{drd}^{\tonehlf} \vx) \trp \Sigma (C_{drd}^{\tonehlf} \vx) \geq 0$, due to the positive semi-definiteness of
$\Sigma$. It is therefore a valid covariance matrix. \figref{drdgm}B
shows a depiction of sampling from the smooth-DRD generative model.
In the following, we will let $\vtheta$ denote the entire hyperparameter set
for the smooth-DRD prior and the noise variance, where
$\vtheta=\{\vtheta_{drd},\delta,\nsevar\}$.

\section{Parameter estimation}
In this section, we describe two methods for inference under the DRD
and smooth-DRD priors: (1) empirical Bayesian inference via evidence
optimization using the Laplace approximation; and (2) fully Bayesian
inference via MCMC sampling. The first seeks to find the MAP estimate
of the latent vector $\vu$ governing region sparsity via optimization
of the log marginal likelihood, and then provides a conditional MAP
estimate of the weights $\vw$. The second uses MCMC sampling to
integrate over $\vu$ and provides the posterior mean of $\vw$ given
the data via an average over samples.

 \subsection{Empirical Bayes inference with Laplace approximation}
 The likelihood $p(\vy|X,\vw,\nsevar)$ (eq.~\ref{ll}) and the prior
 $p(\vw|\vu,\vtheta_{drd},\delta)$ (eq.~\ref{wprior}) are both Gaussian given the latent
 variables $\vu$ and hyperparameters $\vtheta$, giving a conditionally Gaussian posterior over the
 regression weights:
\begin{equation} \label{eq:condMAPw}
p(\vw|X, \vy, \vu, \vtheta) = \Nrm(\vmu_{\vw}, \Lambda_{\vw}), 
\end{equation}
with covariance and mean given by
\begin{equation}\label{eq:muw}
\Lambda_{\vw} = (\tfrac{1}{{\nsevar}}X^\top X + C^{-1})^{-1}, \quad
\vmu_{\vw} = \tfrac{1}{{\nsevar}} \Lambda_{\vw} X^\top \vy, 
\end{equation}
where prior covariance matrix $C$ is a function of $\vu$ and
$\vtheta$. The posterior mean $\vmu_{\vw}$ is also the MAP estimate
of $\vw$ given latent vector $\vu$ and hyperparameters $\vtheta$.


Empirical Bayes inference involves setting the hyperparameters by
maximizing the marginal likelihood or evidence, given by
\begin{equation} \label{eq:evidence}
p(\vy|X,\vtheta) = \iint p(\vy|X,\vw,\sigma^2) p(\vw|\vu,\delta) p(\vu|\vtheta_{drd})\, d\vw\, d\vu.
\end{equation}
We can take the integral over $\vw$ analytically due to the
conditionally Gaussian prior and the likelihood, giving the simplified
expression 
\begin{equation} \label{eq:evidence2}
 p(\vy|X,\vtheta) = \int p(\vy|X,\vu,\sigma^2,\delta)\, p(\vu|\vtheta_{drd})\,
 d\vu,
\end{equation}
where the conditional evidence given $\vu$ is a normal density
evaluated at $\vy$,
\begin{equation} \label{eq:condev}
p(\vy|X,\vu,\vtheta) = \Nrm(\vy| \mathbf{0}, XCX\trp + \sigma^2 I).
\end{equation}
However, the integral over $\vu$
has no analytic form. We therefore resort to the Laplace's method to
approximate this integral. 

\subsubsection{Laplace approximation}

Laplace's method provides a technique for approximating intractable
integrals using a second-order Taylor expansion in $\vu$ of the log of the
integrand in (eq.~\ref{eq:evidence2}). This method is equivalent to
approximating the posterior over $\vu$ given $\vtheta$ by a Gaussian
centered on its mode (\cite{MacKayBook}, chap.~27). The exact posterior is given by Bayes' rule:
\begin{equation} \label{eq:upost}
p(\vu|X,\vy,\vtheta) = \frac{1}{Z}\, p(\vy|X,\vu,\sigma^2,\delta)p(\vu|\vtheta_{drd}), 
\end{equation}
where the normalizing constant, $Z = p(\vy|X,\vtheta)$, is the
marginal likelihood we wish to compute. The Gaussian approximation to
the posterior is
\begin{equation} \label{eq:laplapprox}
p(\vu|X,\vy,\vtheta) \approx \Nrm(\vm_{\vu}, \Lambda_\vu),
\end{equation}
where $\vm_{\vu}$ is the posterior mode and $\Lambda_\vu$ is a local
approximation to the posterior covariance. Substituting this
approximation into (eq.~\ref{eq:upost}), we can directly solve for $Z$:
\begin{equation} \label{eq:laplev}
Z \approx \frac{p(\vy|X,\vu,\sigma^2,\delta)p(\vu|\vtheta_{drd})}{\Nrm(\vm_{\vu}, \Lambda_\vu)}.
\end{equation}
The right-hand-side of this expression can be evaluated at any $\vu$,
but it is conventional to use the mode, $\vu = \vm_{\vu}$, given that
this is where the approximation is most accurate.


To compute the Laplace approximation, we first numerically optimize the
log of the posterior (eq.~\ref{eq:upost}) to find its mode:
\begin{equation} \label{eq:mu}
\vm_{\vu} 
= \arg\max_\vu \; \Big[ \log p(\vy|X,\vu,\nsevar,\delta) + \log p(\vu|\vtheta_{drd}) 
\Big],
\end{equation}
where the first term is the log of the conditional evidence given $\vu$
(eq.~\ref{eq:condev}),
\begin{align}\label{eq:datalike1}
 \log p(\vy|X,\vu,\nsevar,\delta) &= -\frac{1}{2}\log|XCX^\top +\nsevar
 I|-\frac{1}{2}\vy^\top(XCX^\top +\nsevar I)^{-1}\vy + const,
\end{align}
 and the second is the log of the GP prior for $\vu$,
 \begin{align}\label{eq:datalike2}
 \log p(\vu|\vtheta_{drd}) &= -\frac{1}{2}(\vu-b\vones)^\top K^{-1}(\vu-b\vones)-\frac{1}{2}\log |K|+const.
\end{align}
We use quasi-Newton methods to optimize this objective function because
the fixed point methods developed for ARD (e.g., \cite{mackay1992Bayesian,tipping2003fast}), which operate on one
element of the prior precision vector at a time, are inefficient due to the strong
dependencies induced by the GP prior. However, because this
high-dimensional optimization problem is non-convex, we also formulate a novel
approach for optimizing $\vu$ using a two-stage convex relaxation
inspired by \cite{Wipf08}. We will present the method in Sec. 4.1.2.


Given the mode of the log-posterior $\vm_{\vu}$, the second step to
computing the Laplace-based approximation to the marginal likelihood
is to compute the Hessian (2nd derivative matrix) of the
log-posterior at $\vm_{\vu}$. The negative inverse of the Hessian gives us the posterior covariance
for the Laplace approximation (eq.~\ref{eq:laplapprox}):
\begin{equation} \label{eq:Lam}
 \Lambda_{\vu} = \left(-\frac{\partial^2}{\partial \vu \partial
 \vu^\top} \Big[ \log p(\vy| X,\vu,\nsevar,\delta) + \log p(\vu| \vtheta_{drd}) \Big] \right)^{-1}.
\end{equation}
See Appendix \ref{hess} for the explicit derivation of Hessian for the DRD model. 




Given these ingredients, we can now write down the 
approximation to the log marginal likelihood (eq.~\ref{eq:laplev}):
\begin{equation} 
 \log p(\vy|X, \vtheta) \approx 
\log p(\vy|X, \vm_{\vu},\nsevar,\delta) 
 + \log p(\vm_{\vu}|\vtheta_{drd}) + \tfrac{1}{2} \log |\Lambda_{\vu}|
 +const, \label{eq:lm}
\end{equation}
where the first term is simply the log conditional evidence
(eq.~\ref{eq:datalike1}) with prior covariance $C$ evaluated at
$\vm_\vu$.

It is this log-marginal likelihood that we seek to optimize in order
to learn hyperparameters $\vtheta$. The key difficulty is that the
Laplace approximation parameters $\vm_{\vu}$ and $\Lambda_{\vu}$
depend implicitly on $\vtheta$ (since $\vm_{\vu}$ is determined by
numerical optimization at a fixed value of $\vtheta$), making it
impractical to evaluate their derivatives with respect to $\vtheta$.
To address this problem, we introduce a method for partially
decoupling the Laplace approximation from the hyperparameters
(Sec. 4.1.3).

\subsubsection{A two-stage convex relaxation to Laplace Approximation}
The optimization for $\vm_\vu$ (eq.~\ref{eq:mu}), the mode of the
posterior over the latent vector $\vu$, is a critical step for
computing the Laplace approximation. However, the negative
log-posterior is a non-convex function in $\vu$, meaning that there is
no guarantee of obtaining the global minimum. In this section, DRD resembles
the original ARD model. Neither of the two most popular optimization
methods for ARD, MacKay's fixed-point method
\citep{mackay1992Bayesian} and Tipping and Faul's fast-ARD
\citep{tipping2003fast}, are guaranteed to converge to a local minimum
or even a fixed point of the log-posterior.



In this section, we introduce an alternative formulation of the cost
function in (eq.~\ref{eq:mu}) 
using an auxiliary function: this provides a tight convex upper bound
that can be optimized more easily. The technique is similar to the
iterative re-weighted $l_1$ formulation of ARD in \cite{Wipf08}.

Let $\mathcal{L}(\vu)$ denote the sum of terms in the negative log-posterior
(eq.~\ref{eq:mu}) that involve $\vu$,
\begin{equation} 
 \mathcal{L}(\vu) =\frac{1}{2}\log|XCX^\top +\nsevar I|+\frac{1}{2}\vy^\top(XCX^\top +\nsevar I)^{-1}\vy
 +\frac{1}{2}(\vu-b\vones)^\top K^{-1}(\vu-b\vones),
 \end{equation} 
 where $C=\diag(e^\vu)$. We denote the three terms it contains as:
 \begin{eqnarray}
 \mathcal{L}_1(\vu) &=&\frac{1}{2} \log |X \diag(e^\vu)X^\top+\nsevar I|\\
 \mathcal{L}_2(\vu) &=&\frac{1}{2} \vy^\top (X\diag(e^\vu)X^\top+\nsevar I)^{-1}\vy\\
 \mathcal{L}_3(\vu) &=& \frac{1}{2}(\vu-b\vones)^\top K^{-1}(\vu-b\vones).
\end{eqnarray} 
Here $\mathcal{L}_1(\vu)$ and $\mathcal{L}_3(\vu)$ are both convex in
 $\vu$ (see proof in Appendix \ref{conv}). We can derive a tight convex upper bound for
 $\mathcal{L}_2(\vu)$, thus providing a tight convex upper bound for $\mathcal{L}(\vu)$.

 We know that $\mathcal{L}_2(\vu)$ is non-convex, but we are
 interested in rewriting it using concave duality. Let
 $\vh(\vu): \mathbb{R}^p\rightarrow \Omega \subset \mathbb{R}^p$ be a
 mapping with range $\Omega$, which may or may not be a one-to-one
 map. We assume that there exists a concave function $\Phi(\veta):\Omega\rightarrow\mathbb{R}, \forall \veta\in \Omega$, such that $\mathcal{L}_2(\vu)=\Phi(\vh(\vu))$ holds. To exploit this technique, we first rewrite $\mathcal{L}_2$
 using the
 matrix inverse lemma \citep{higham2002accuracy} as:
\begin{eqnarray}
\mathcal{L}_2(\vu) &=& \frac{1}{2\nsevar} \vy^\top\vy-\frac{1}{2\sigma^4}\vy^\top X \left (\frac{1}{\nsevar}X^\top X+\diag(e^{-\vu})\right)^{-1}X^\top\vy.
\end{eqnarray} 
Then, setting $\vh(\vu)=e^{-\vu}$, which is convex in $\vu$, we have
\begin{eqnarray}
\mathcal{L}_2(\vu)=\Phi(\vh(\vu))=\frac{1}{2\nsevar} \vy^\top\vy-\frac{1}{2\sigma^4}\vy^\top X\left(\frac{1}{\nsevar}X^\top X+\diag(\vh(\vu))\right)^{-1}X^\top\vy.
\end{eqnarray} 
This expression is concave in $\vh(\vu)$ (inverse of a matrix is convex), and thus can be expressed as a minimum over upper-bounding hyperplanes via
\begin{eqnarray}
\mathcal{L}_2(\vu) =\Phi(\vh(\vu))=\mbox{inf}_{\vz\in \mathbb{R}^p}\left[\vz^\top \vh(\vu)-\mathcal{L}_\vh^*(\vz)\right],
\end{eqnarray} 
where $\mathcal{L}_\vh^*(\vz)$ is the concave conjugate of $\Phi(\veta)$ that is defined by the duality relationship
\begin{eqnarray}\label{concavebound}
\mathcal{L}_\vh^*(\vz) =\mbox{inf}_{\veta\in
 \mathbb{R}^p}\left[\vz^\top
 \veta-\Phi(\veta)\right], 
\end{eqnarray} 
and $\vz$ is the dual variable. Note, however, that for our purpose it
is not necessary to ever explicitly compute
$\mathcal{L}_\vh^*(\vz)$. This leads to the following upper-bounding
auxiliary cost function
\begin{eqnarray}\label{auxbound}
\Phi(\vh(\vu),\vz) =\vz^\top\vh(\vu)-\mathcal{L}_\vh^*(\vz) \ge \Phi(\vh(\vu)).
\end{eqnarray} 
Thus, it naturally admits the tight convex upper bound for $\mathcal{L(\vu)}$,
\begin{eqnarray}\label{ub}
\mathcal{L}(\vu,\vz) \overset{\Delta}{=} \vz^\top\vh(\vu)-\mathcal{L}_\vh^*(\vz)+\mathcal{L}_1(\vu)+\mathcal{L}_3(\vu)\ge \mathcal{L(\vu)}. 
\end{eqnarray} 
Moreover, for any fixed $\veta=\vh(\vu)$, it's well-known that the
minimum of the right hand side of (eq.~\ref{concavebound}) is achieved
at
\begin{eqnarray}\label{zopt}
\hat{\vz}=\nabla_{\veta}\Phi(\veta)|_{\veta=\vh(\vu)}.
\end{eqnarray} 
This leads to the general optimization procedure presented in
Algorithm \ref{tab1}. By repeatedly refining the dual parameter $\vz$,
we can obtain a repeatedly improved convex relaxation, leading to a
solution superior to that of the initial convex relaxation.
\begin{algorithm}[!t]
\caption{A two-stage convex relaxation method for DRD Laplace approximation}
\centering
\begin{algorithmic} 
\STATE\textbf{Input:} $X,\vy,\vtheta=\{\nsevar,\delta,b,\rho,l\}$
\STATE\textbf{Output:} $\hat{\vu}$
\STATE initialize dual variable $\hat{\vz}_i$ = 1, $\forall i=1,2,...,p$
\STATE Repeat the following two steps until convergence:
\STATE 1. Fix $\hat{\vz}$, let $\hat{\vu} = \mbox{argmin}_{\vu\in \mathbb{R}^p}\;\left[ \vz^\top\vh(\vu)+\mathcal{L}_1(\vu)+\mathcal{L}_3(\vu)\right]$ in (eq.~\ref{ub})
\STATE 2. Fix $\hat{\vu}$, let $\hat{\vz}=\nabla_{\veta}\Phi(\veta)|_{\veta=\vh(\hat{\vu})}$ in (eq.~\ref{zopt})
\end{algorithmic}\label{tab1}
\end{algorithm}


Now we show the analysis of global convergence. According to the Zangwill's \textsl{Global Convergence Theorem} \citep{zangwill1969nonlinear}, let
$\mathcal{A}(\cdot): \mathcal{U}\rightarrow \mathcal{P}(\mathcal{U})$ be a point-to-set mapping to handle the
multi-global minima case, which satisfies Steps 1 and 2 of the
proposed algorithm, then

\begin{thm}
From any initialization point
$\vu^0 \in \mathbb{R}^p$, the sequence of parameter estimates
$\{\vu^k\}$ generated via $\vu^{k+1} \in \mathcal{A}(\vu^k)$ is
guaranteed to converge monotonically to a local minimum (or saddle
point) of $\mathcal{L}(\vu)$. 
\end{thm}

\begin{proof}
Let $\Gamma\in\mathcal{U}$ be a solution set. In order to use the global convergence theorem, we need to show that \\
1) all points $\{\vu^k\}$ are contained in a compact set $S \in\mathcal{U}$, where $\mathcal{U}$ is $\mathbb{R}^p$ in our problem; \\
2) there is a continuous function $Z$ on $\mathcal{U}$ such that\\
(a) if $x\not\in\Gamma$, then $Z(y) < Z(x)$ for all $y \in \mathcal{A}(x)$;\\
(b) if $x\in\Gamma$, then $Z(y) \le Z(x)$ for all $y \in \mathcal{A}(x)$;\\
3) the mapping $\mathcal{A}$ is closed at points outside $\Gamma$.

First, let's define the mapping $\mathcal{A}$ to be achieved by 
\begin{eqnarray}\label{mapA}
\vu^{k+1} \in \mathcal{A}(\vu^k)= \mbox{argmin}_{\vu\in \mathbb{R}^p}\;\mathcal{F}(\vu,\vu^k)=\mbox{argmin}_{\vu\in \mathbb{R}^p}\;{\vz^k}^\top\vh(\vu)-\mathcal{L}_\vh^*(\vz)+\mathcal{L}_1(\vu)+\mathcal{L}_3(\vu),
\end{eqnarray}
where $\vz^k=\nabla_{\veta}\Phi(\veta)|_{\veta=\vh(\vu^k)}$. 
We can prove that $\mathcal{F}$ is coercive, i.e., when $||\vu||\rightarrow\infty$, we have $\mathcal{F}(\vu) \rightarrow \infty$ (proof in Appendix \ref{bound}). Therefore, the solution set of $\mathcal{F}(\vu)$ is bounded and nonempty. Accordingly, $\mathcal{A}(\vu)$ is nonempty. Using Proposition 7 in \citep{gunawardana2005convergence}, we can further show that the point-to-set mapping $\mathcal{A}$ is closed at $\vu\in\mathcal{U}$. Condition 3 is satisfied. 

For each $\vu^{k}$, $\vu^{k+1}$ is the solution of $\mathcal{F}(\vu)$, and $\mathcal{A}(\vu)$ is a closed mapping; therefore each $\vu^{k+1}$ belongs to a compact set. We know that the union of two compact sets is compact. Therefore, all points $\{\vu^k\}$ are contained in a compact set $S \in\mathcal{U}$. Condition 1 is satisfied. 

To prove condition 2, we must show that for any $\vu^k$, $\mathcal{L}(\vu^{k+1})<\mathcal{L}(\vu^{k})$ for all $\vu^{k+1}\in\mathcal{A}(\vu^k)$ if $\vu^k\not\in\Gamma$; $\mathcal{L}(\vu^{k+1})\le\mathcal{L}(\vu^{k})$ for all $\vu^{k+1}\in\mathcal{A}(\vu^k)$ if $\vu^k\in\Gamma$. At any $\vu^k$, the auxiliary cost function $\mathcal{F}(\vu)$ (eq.~\ref{mapA}) is strictly tangent to $\mathcal{L}(\vu)$ at $\vu^k$. Therefore, if $\vu^k\not\in\Gamma$, $\mathcal{L}(\vu^k)=\mathcal{F}(\vu^k)>\mathcal{F}(\vu^{k+1})\ge\mathcal{L}(\vu^{k+1})$, thus $\mathcal{L}(\vu^k)>\mathcal{L}(\vu^{k+1})$; if $\vu^k\in\Gamma$, $\mathcal{L}(\vu^k)=\mathcal{F}(\vu^k)=\mathcal{F}(\vu^{k+1})\ge\mathcal{L}(\vu^{k+1})$, thus $\mathcal{L}(\vu^k)\ge\mathcal{L}(\vu^{k+1})$. Condition 2 is satisfied.
\end{proof}
The algorithm could theoretically converge to a saddle point, but any minimal perturbation would easily lead to escape.

\subsubsection{Decoupled Laplace approximation}


To optimize the marginal likelihood for the DRD hyperparameters
(eq.~\ref{eq:lm}), we should ideally replace $\vm_\vu$ and
$\Lambda_\vu$ with explicit expressions in $\vtheta$ in order to
accurately compute derivatives with respect to $\vtheta$. However, the
deterministic formulation of such functions is intractable. We can
nevertheless partially overcome this dependence by introducing a
``decoupled'' Laplace approximation that takes account of the
dependence of $\Lambda_\vu$ on the hyperparameters $\vtheta_{drd}$. \cite{wu2017gaussian} also proposed a conceptually similar decoupled Laplace approximation.

Specifically, we rewrite the inverse Laplace posterior covariance
(eq.~\ref{eq:Lam}):
\begin{equation} \label{eq:Lamde}
 \Lambda_{\vu} = \left(\Gamma+\Psi(\vtheta_{drd}) \right)\inv
\end{equation}
where $\Gamma$ is the negative Hessian of the log-likelihood (which is
independent of $\vtheta_{drd}$), 
\begin{equation}
 \label{eq:Gam}
 \Gamma=-\frac{\partial^2}{\partial \vu \partial \vu^\top} \log p(\vy|X,\vu,\nsevar,\delta),
\end{equation}
 and $\Psi(\vtheta_{drd})$ is the
precision matrix of the prior distribution for $\vu$,
\begin{equation}
 \label{eq:Psi}
 \Psi(\vtheta_{drd})=-\frac{\partial^2}{\partial \vu \partial
 \vu^\top} \log p(\vu|\vtheta_{drd}) = K\inv,
\end{equation}
which is the inverse of the GP prior covariance governing $\vu$
(eq.~\ref{eq:K}).
Substituting for $\Lambda_\vu$ in (eq. \ref{eq:lm}), this gives:
 \begin{eqnarray} \label{eq:delapl}
 \log p(\vy|X, \vtheta) \approx \log p(\vy|X, \vm_{\vu},\nsevar,\delta) 
 + \log p(\vm_{\vu}|\vtheta_{drd})- \tfrac{1}{2} \log | \Gamma+K\inv|+const .
\end{eqnarray}
This form decomposes the curvature at the posterior mode into the
likelihood curvature and the prior curvature. In this way, the posterior
curvature tracks the influence of the change in the prior curvature as we
optimize the hyperparameters $\vtheta$, while keeping the influence of
the likelihood curvature fixed. This decoupling allows us to update
the posterior without recomputing the Hessian. It will be accurate so
long as the Hessian of the likelihood changes slowly over local
regions in parameter space.

To optimize hyperparameters under the decoupled Laplace approximation,
we fix $\vm_\vu$ and $\Gamma$ using the current mode of the posterior,
and optimize (eq.~\ref{eq:delapl}) directly for $\vtheta$,
incorporating the dependence of $K$ on $\vtheta_{drd}$. With this
approach, the first term, $\log p(\vy|X, \vm_{\vu},\nsevar,\delta)$,
captures the dependence on $\nsevar$ and $\delta$; the second term,
$\log p(\vm_{\vu}|\vtheta_{drd})$, restricts $\vtheta_{drd}$ around
the current mode; and the third term
$- \tfrac{1}{2} \log | \Gamma+K\inv|$ pushes $\vtheta_{drd}$ along the
second order curvature given the GP kernel. This decoupling weakens
the strong dependency between $\vtheta_{drd}$ and $\vm_\vu$,
maintaining the accuracy of the Laplace approximation as we adjust
$\vtheta_{drd}$.

To ensure the accuracy of the Laplace approximation, in each iteration $t$,
we optimize eq.~(\ref{eq:delapl}) over a restricted region of the hyperparameter space
around the previous hyperparameter setting $\theta^{t-1}$, which allows varying within 20\% of its current value on each iteration in our experiments. This
prevents $\vtheta$ from moving too far from the region where the
current Laplace approximation ($\vm_\vu$ and $\Gamma$) is accurate.
Then, based on a new hyperparameter setting $\vtheta^t$, we update the
Laplace approximation parameters $\vm_\vu$ and $\Gamma$. This
procedure is summarized in Algorithm \ref{A1}. The algorithm stops
when $\{\vm_{\vu},\Gamma\}$ and $\vtheta$ converge. The empirical
Bayes estimate is then given by the MAP estimate of the weights
$\vw_{map}=\vmu_\vw$ (eq. \ref{eq:muw}) conditioned on the optimal
latents $\hat{\vu}=\vm_{\vu}$ and hyperparameters $\hat{\vtheta}$.



\begin{algorithm}[!t]
\caption{Evidence optimization using decoupled Laplace approximation}
\centering
\begin{algorithmic} 
\STATE\textbf{Input:} $X,\vy$
\STATE\textbf{Output:} latents $\hat{\vu}$, hyperparameters $\hat{\vtheta}=\{\hat{\nsevar},\hat{\delta},\hat{b},\hat{\rho},\hat{l}\}$.
\STATE At iteration $t$:
\STATE 1. Numerically optimize log-posterior for latents $\vm_{\vu}^{t}$
using (eq.~\ref{eq:mu}) or Algorithm \ref{tab1}.
\STATE 2. Compute $\Gamma^t$ using negative Hessian of the log
conditional evidence (eq.~\ref{eq:Gam}).
\STATE 3. Numerically optimize
$p(\vy|X,\vtheta,\vm_{\vu}^{t},\Gamma^{t})$
(eq. \ref{eq:delapl}) for $\vtheta^t$.
\STATE Repeat step 1, 2 and 3 until $\{\vm_{\vu},\Gamma\}$
and $\vtheta$ converge.
\end{algorithmic}\label{A1}
\end{algorithm}

\subsection{Fully Bayesian inference with MCMC}
An alternate approach to the empirical Bayesian inference procedure
described above is to perform fully Bayesian inference using Markov
Chain Monte Carlo (MCMC). Using sampling, we can compute the
integrals over $\vu$ and
$\vtheta$ in order to compute the posterior mean (Bayes' least squares
estimates) for $\vw$. The full posterior distribution over
$\vw$ can be written as
\begin{eqnarray}
p(\vw|X,\vy) &=& \iint p(\vw | X,\vy,\vu,\vtheta) p(\vu,\vtheta|X,\vy) \,
 d\vu\, d\vtheta \\
&=& \iint \Nrm(\vw | \vmu_\vw, \Lambda_\vw ) p(\vu,\vtheta|X,\vy) \,
 d\vu\, d\vtheta,
\end{eqnarray}
where mean $\vmu_\vw$ and covariance $\Lambda_\vw$ are functions of
$\vu$ and $\vtheta$ (eq.~\ref{eq:muw}). This suggests a Monte Carlo representation of
the posterior as
\begin{align}
p(\vw|X,\vy) &= \frac{1}{N} \sum_{i=1}^N \Nrm \left(\vw \; \Big|\; 
 \vmu_\vw(\vu^{(i)},\vtheta^{(i)}) ,
 \Lambda_\vw(\vu^{(i)},\vtheta^{(i)}) \right) \\
\vu^{(i)},\vtheta^{(i)} &\sim p(\vu,\vtheta|X,\vy),
\end{align}
where $i$ is the index of the samples and $N$ is the total number of samples.
We can use Gibbs sampling to alternately sample $\vu$ and
$\vtheta$ from their conditional distributions given the other. The
joint posterior distribution of $\vu$ and
$\vtheta$ has the following proportional relationship,
 \begin{eqnarray} \label{post_utheta}
p(\vu,\vtheta|X,\vy)\propto p(\vy|X,\vu,\nsevar,\delta)p(\vu|\vtheta_{drd})\mbox{Prior}(\vtheta),
\end{eqnarray}
 where $p(\vy|X,\vu,\nsevar,\delta)$ and $p(\vu|\vtheta_{drd})$ have
 the likelihoods given in (eq.~\ref{eq:datalike1}) and (eq.~\ref{eq:datalike2}), and $\mbox{Prior}(\vtheta)$ is the prior distribution for $\vtheta$.

 \textbf{Sampling latents $\vu | \vtheta$}\\
 The first phase of Gibbs sampling is to sample $\vu$ from the
 conditional distribution of $\vu$ given $\vtheta$,
 \begin{eqnarray} \label{post_utheta}
\vu|\vtheta \sim p(\vy|X,\vu,\nsevar,\delta)p(\vu|\vtheta_{drd}).
\end{eqnarray}
This is the product of a Gaussian process prior $p(\vu|\vtheta_{drd})$
and a likelihood function $p(\vy|X,\vu,\vtheta)$ that ties the latent
variables $\vu$ to the observed data. This setting meets the
requirements of elliptical slice sampling (ESS), a rejection-free MCMC
\citep{murray2009elliptical}. ESS generates random elliptical loci
using the Gaussian prior and then searches along these loci to find
acceptable points by evaluating the data likelihood. This method takes
into account strong dependencies imposed by GP covariance on the
elements of the vector $\vu$ to facilitates faster mixing. It also
requires no tuning parameters, unlike alternative samplers such as
Metropolis-Hastings or Hamiltonian Monte Carlo, but performs similarly
to the best possible performance of a related M-H scheme. To overcome
slow mixing that can result when the prior covariance is highly
elongated, we apply ESS to a whitened variable using a
reparametrization trick, discussed in more details in Sec. \ref{sec:wGP}.


\textbf{Sampling hyperparameters $\vtheta|\vu$}\\
The conditional distribution for sampling $\vtheta$ given $\vu$ is
 \begin{eqnarray}
\vtheta|\vu\sim p(\vy|X,\vu,\nsevar,\delta)p(\vu|\vtheta_{drd})\mbox{Prior}(\vtheta),
\end{eqnarray}
where $\vtheta = \{\nsevar, b,\rho,l,\delta\}$ contains five
individual hyperparmaters. We therefore perform slice sampling for
each variable conditioned on the others. We use prior distributions
of the form:
\begin{eqnarray}\label{hyphyp}
\log(\nsevar)\sim \mathcal{N}(m_n, \sigma_n^2),\mbox{\quad}
b\sim \mathcal{N}(m_b, \sigma_b^2),\mbox{\quad}
\rho\sim {\Gamma}(a_\rho, b_\rho),\mbox{\quad}
l\sim {\Gamma}(a_l, b_l),\mbox{\quad}
\delta\sim{\Gamma}(a_{\delta}, b_{\delta}).
\end{eqnarray}
We put a Gaussian prior on the log of $\nsevar$ instead of $\nsevar$. We will provide the values for these priors in Sec. \ref{sec:synth} on synthetic experiments. To control the number of samples, we inspect burn-in of MCMC, e.g., the training error and the change of coefficient given the averaged coefficient samples. 

\subsection{ Whitening the GP prior using reparametrization}
\label{sec:wGP}

In both Laplace approximation and MCMC frameworks, the latent vector
$\vu$ depends on the product of the conditional evidence
$p(\vy|X,\vu,\nsevar,\delta)$ and the GP prior
$p(\vu|\vtheta_{drd})$. The GP prior (which is the primary difference
between our model and standard ARD) introduces strong dependencies
between $\vu$ and GP hyperparameters, resulting in a highly elliptical
joint distribution. Such distributions are often problematic for both
optimization and sampling. For example, if we are trying to perform
Gibbs sampling on $\vu$ and the GP length scale hyperparameter
$l$, and the prior is strong relative to the evidence term, the samples $\vu | l^{(i)}$ will have smoothness strongly determined by
$l^{(i)}$, and the samples $l | \vu^{(i)}$ will in turn be strongly
determined by the smoothness of the current sample $\vu^{(i)}$. In this case,
mixing will be slow, and Gibbs sampling will take a long time to
explore the full posterior over different values of $l$.


We can overcome this difficulty with a technique known as the
``reparametrization trick,'' which involves reparameterizing the model
so that the unknown variables are independent under the prior
\citep{murray2010slice}. If we have prior $P(\vu) = \mathcal{N}(b\vones,K)$, then 
$\vu$ can be described equivalently by a deterministic transformation of
a standard normal random variable $\vv$:
\begin{eqnarray}
\label{wp}
\vv\sim \mathcal{N}(0, I ),\mbox{\quad}
\vu=L\vv+b\vones,
\end{eqnarray}
where $K = LL\trp$ is the Cholesky factorization of prior covariance
$K$. 

This reparametrization simplifies Laplace-approximation-based
inference by allowing a change of variables in (eq.~\ref{eq:mu}) so
that we directly maximize
$p(\vy|X,\vv,\vtheta)\mathcal{N}(\vv|0,I)$ for $\vv$. This
optimization problem has better conditioning, and eliminates the
computational problem of computing $\vu \trp K\inv \vu$ in the log
prior, which is replaced by a simple ridge penalty of the form
$\vv\trp\vv$.


For sampling-based inference, the reparametrization allows us to
improve mixing performance because the conditionals $\vv|\vtheta$ and
$\vtheta|\vv$ exhibit much weaker dependencies than $\vu|\vtheta$ and
$\vtheta|\vu$. Moreover, elliptical slice sampling for $\vv|\vtheta$ is
more efficient because it involves loci on a sphere instead of a
highly elongated ellipsoid.

\subsection{Fourier dual form}\label{fd}
 
A second trick for improving the computational performance of DRD is
to perform optimization of the latent variable $\vu$ (or $\vv$) in the
Fourier domain. When the GP prior induces a high degree of smoothness
in $\vu$, the prior covariance $K$ becomes approximately low rank,
meaning that it has a small number of non-negligible eigenvalues.
Because the covariance function (eq.~\ref{eq:K}) is shift-invariant,
the eigenspectrum of $K$ has a diagonal representation in the Fourier
domain, a consequence of Bochner's theorem \citep{Stein99,Lazaro10}.
We can exploit this representation to optimize $\tilde{\vu}$ (the
discrete Fourier transform of $\vu$) while ignoring Fourier components
above a certain high-frequency cutoff, where this cutoff depends on
the length scale $l$. This results in a lower-dimensional
optimization problem. Fourier-domain representation of the latent vector $\vu$ also
simplifies the application of the reparametrization trick described
above because the Cholesky factor $L$ is now a diagonal matrix that can be
computed analytically from the spectral density of the
squared-exponential prior.

To summarize the joint application of the reparametrization and Fourier dual
tricks in our model, they can be understood as allowing us to
draw samples $\vu \sim \mathcal{N}(b\vones,K)$ via the series of transformations:
\begin{alignat}{2}
\tilde \vv &\sim \Nrm(0,I), \qquad & & \textrm{\it whitened Fourier domain
 sample}\label{vt} \\
\tilde \vu &= L \tilde \vv + \tilde \vb, & &\textrm{\it transformed Fourier
 domain sample}\\
\vu &= B \tilde \vu, & &\textrm{\it inverse Fourier transform}
\end{alignat}
where $\tilde \vb$ is the discrete Fourier transform of $b \vones$, a
vector of zeros except for a single non-zero element carrying the DC
component, and $B$ is the truncated (tall skinny) discrete inverse
Fourier transform matrix mapping the low-frequency Fourier components
represented in $\tilde \vu$ to the space domain.


Note that the smoothness on $\vu$, which controls the spatial scale of
dependent sparsity, is different from the smoothing prior used in
smooth-DRD to induce smoothness in the coefficients $\vw$, although
both can benefit from sparse Fourier-domain representation in cases
where the relevant length scale is large.


\section{Synthetic experiments}
\label{sec:synth}


\subsection{Simulated example with smooth and sparse weights}

To illustrate and give intuition for the performance of the DRD
estimator, we performed simulated experiments with a vector of regression
weights in a one-dimensional space. We sampled a $p=4000$ dimensional
weight vector $\vw$ from the smooth-DRD prior (see Fig.~1), with
hyperparameters GP mean $b=-8$, GP length scale $l=100$, GP marginal
variance $\rho=36$, smoothness length scale $\delta=50$, measurement
noise variance $\nsevar=5$. We then sampled $n=500$
responses $\vy = X\vw+\mathbf{\epsilon}$, where $X$ is an $n\times p$
design matrix with entries drawn i.i.d. from a standard normal
distribution, and noise $\mathbf{\epsilon} \sim \Nrm(0,5I)$.
%

\begin{figure}[!t] 
 \centering
 \includegraphics[width=0.9\textwidth]{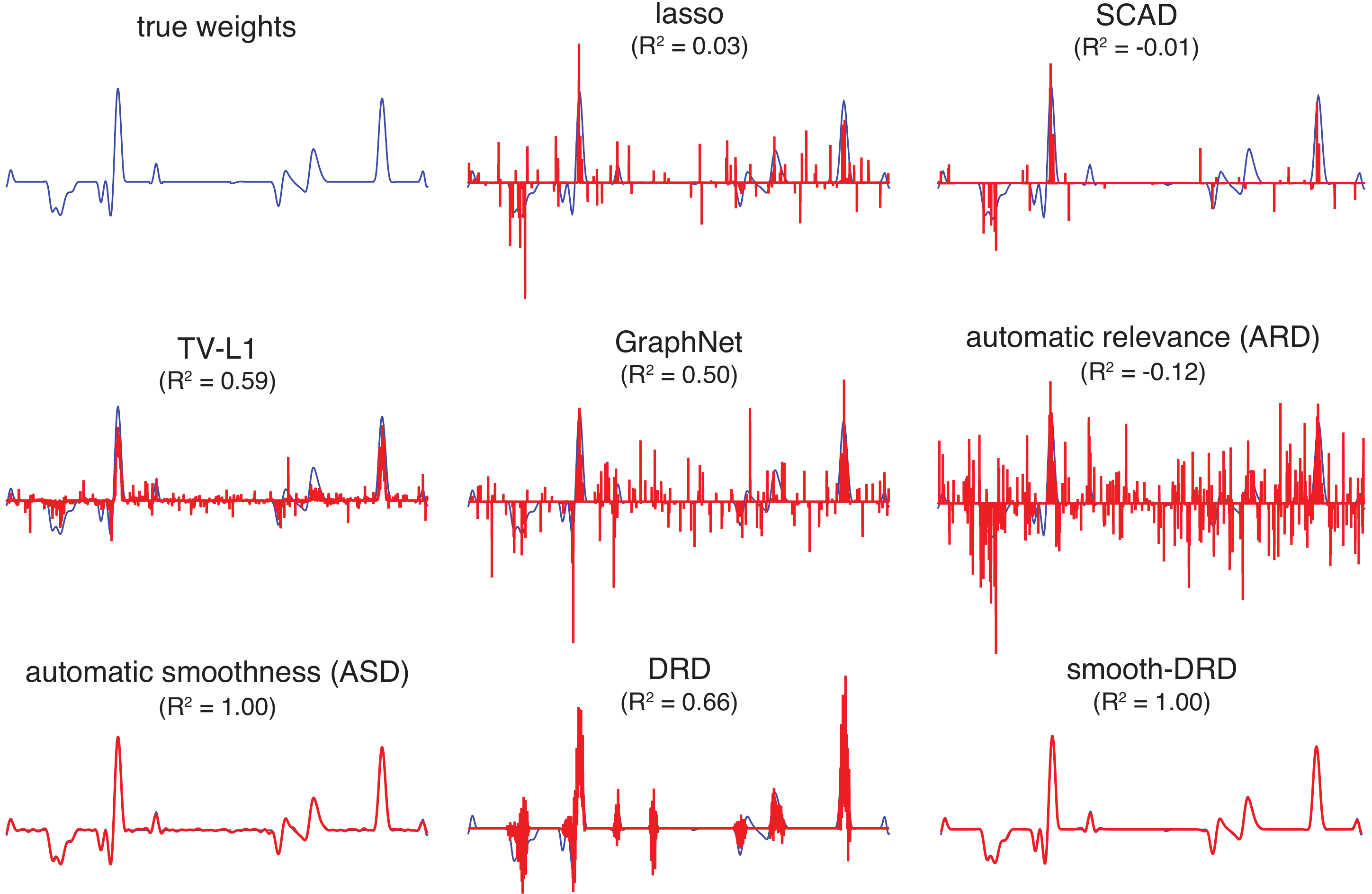} 
 \caption{ Example 4000-element weight vector $\vw$ sampled from
 the smooth-DRD prior (upper left), and estimates obtained from
 different methods on a simulated dataset with $n=500$ samples. The
 $R^2$ performance of each estimate in recovering $\vw$ is indicated
 above each plot. The bottom row shows our estimators: DRD-Laplace
 (bottom center) and smooth-DRD-Laplace (bottom right); the other
 DRD and smooth-DRD estimators (not shown) achieved similar
 performance.}
 \label{fig:ins}
\end{figure}

Fig. \ref{fig:ins} shows an example weight vector drawn from this
prior, along with estimates obtained from a variety of different
estimators:
\begin{itemize}
\item lasso \citep{tibshirani1996regression}, using Least Angle Regression (LARS) implemented
by glmnet\footnote{\url{https://web.stanford.edu/~hastie/glmnet/glmnet_alpha.html}};
\item Automatic Relevance Determination (ARD) \citep{neal1995Bayesian,
 mackay1992Bayesian}, implemented with the classic fixed point algorithm.
\item Automatic Smoothness Determination (ASD)
 \citep{sahani2003evidence}, which uses numerical optimization of
 marginal likelihood to learn the hyperparameters of a squared
 exponential kernel governing $\vw$.
\item Total Variation $l_1$ (TV-L1)
 \citep{michel2011total,baldassarre2012structured,gramfort2013identifying}, combining total variation penalty (also known as fused lasso),
 which imposes an $l_1$ penalty on the first-order differences of
 $\vw$, with a standard lasso penalty.
\item GraphNet \citep{grosenick2011family}, a
graph-constrained elastic net, developed for spatial and temporally
correlated data that yields interpretable model parameters by
incorporating sparse graph priors based on model smoothness or
connectivity, as well as a global sparsity inducing a prior that
automatically selects important variables.
\item Smoothly Clipped Absolute Deviation (SCAD) \citep{fan2001variable}, 
an estimator with non-convex sparsity penalty.

\end{itemize}
%
%

We computed total variation $l_1$ (TV-L1) and graph net (GraphNet)
estimates using the
Nilearn\footnote{\url{http://nilearn.github.io/index.html}} package. SCAD was implemented by SparseReg\footnote{\url{https://github.com/Hua-Zhou/SparseReg}}
\citep{zhou2017sparse}. For lasso, GraphNet, TV-L1 and SCAD, we used cross-validation to set
hyperparameters, whereas ARD and ASD used evidence optimization to
automatically set hyperparameters. For the DRD estimators, we used
evidence optimization to set hyperparameters for Laplace-approximation
based estimates and used sampling to integrate over hyperparameters
for MCMC-based estimates.

For the basic DRD model, which incorporates structured sparsity but
not smoothness, we compared three different inference methods: (1)
Laplace approximation based inference (``DRD-Laplace''); (2) Markov
Chain Monte Carlo (``DRD-MCMC''); and (3) Convex relaxation based
optimization (``DRD-Convex''). Lastly, for the smooth-DRD model, we
used two inference methods: (4) Laplace approximation
(``smooth-DRD-Laplace''); and (5) MCMC (``smooth-DRD-MCMC''). For the
non-MCMC estimators, we initialized the vector of Fourier domain
coefficients $\tilde{\vv}$ (eq.~\ref{vt}) to values of $10^{-3}$ in the first
iteration when learning $\vu$. The hyper-hyperparameters in the MCMC
methods (eq.~\ref{hyphyp}) were set to:
$m_n=-2, \sigma_n^2=5, m_b=-10, \sigma_b^2=8, a_\rho=4, b_\rho=5,
a_l=4, b_l=25, a_{\delta}=4, b_{\delta}=25.$

Fig.~\ref{fig:ins} shows the
reconstruction performance ($R^2$) of the true regression weight
$\vw$ for different estimators. The reconstruction performance
metric for an estimate $\hat \vw$ is given by
$R^2 = 1-\frac{||\vw-\hat{\vw}||_2^2}{||\vw-\bar{\vw}||_2^2}$, where
$||\cdot||_2$ denotes the $l_2$-norm and
$\bar{\vw} = \frac{1}{p} \sum_{i=1}^p \vw_i$ is the mean of vector $\vw$. The true weight vector was sampled from the smooth DRD model. The smooth-DRD estimate achieved the best performance in terms of $R^2$.
The ASD estimate also performed well, although the estimate was not
sparse, exhibiting small wiggles where the coefficients should be
zero. The standard DRD estimate recovered the support of $\vw$ with
high accuracy, but had larger error than smooth-DRD estimates due to
the smoothness of the true $\vw$. The other methods (lasso, ARD, TV-L1, GraphNet and SCAD) all had lower accuracy in recovering both the support
and values of the regression weights.

To provide insight into the performance of ARD, DRD, and smooth-DRD, we plotted the
inferred prior covariance of each model (Fig.~\ref{fig:priorcov}). The DRD and smooth-DRD models were both
similar to ARD in that they achieved sparsity by shrinking the prior
variance of unnecessary coefficients to zero. However, unlike ARD,
their inferred prior covariances both exhibited clusters of non-zero
coefficients, reflecting the dependencies introduced by the latent
Gaussian process. Note also that ARD and DRD covariances were both
diagonal, making the weights independent given the prior variances,
whereas the smooth-DRD covariance had off-diagonal structure that
induced smoothness.

%
%

\begin{figure}[!t] 
 \centering
 \includegraphics[width=0.8\textwidth]{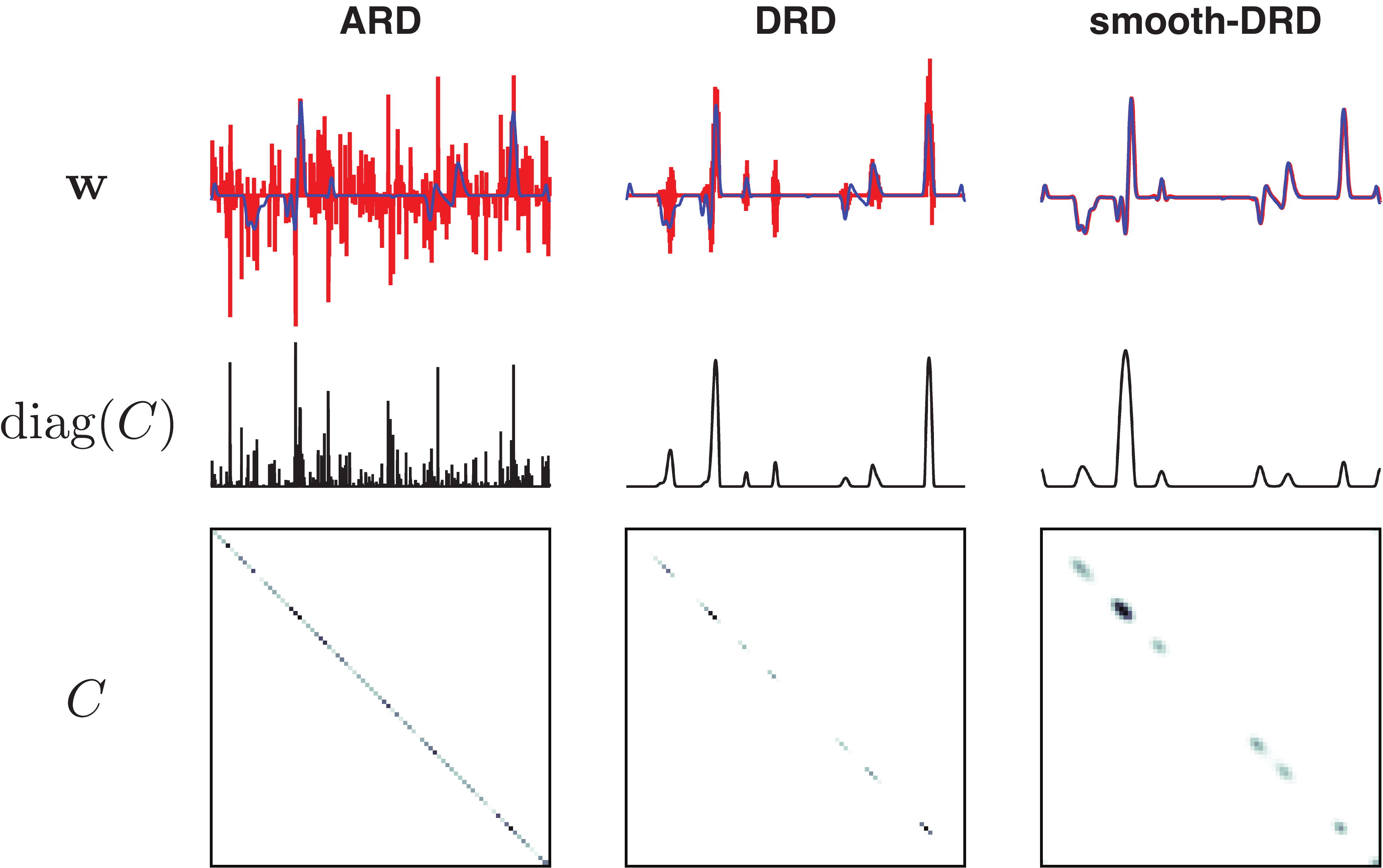} 
 \caption{Estimated filter weights and prior covariances. The upper row shows the true filter (blue) and estimated ones (red); the middle row displays the diagonal of each estimated covariance matrix; and the bottom row shows the entire estimated covariance matrix for each prior.}
 \label{fig:priorcov}
\end{figure}

\begin{figure}[!t] 
 \centering
 \includegraphics[width=0.9\textwidth]{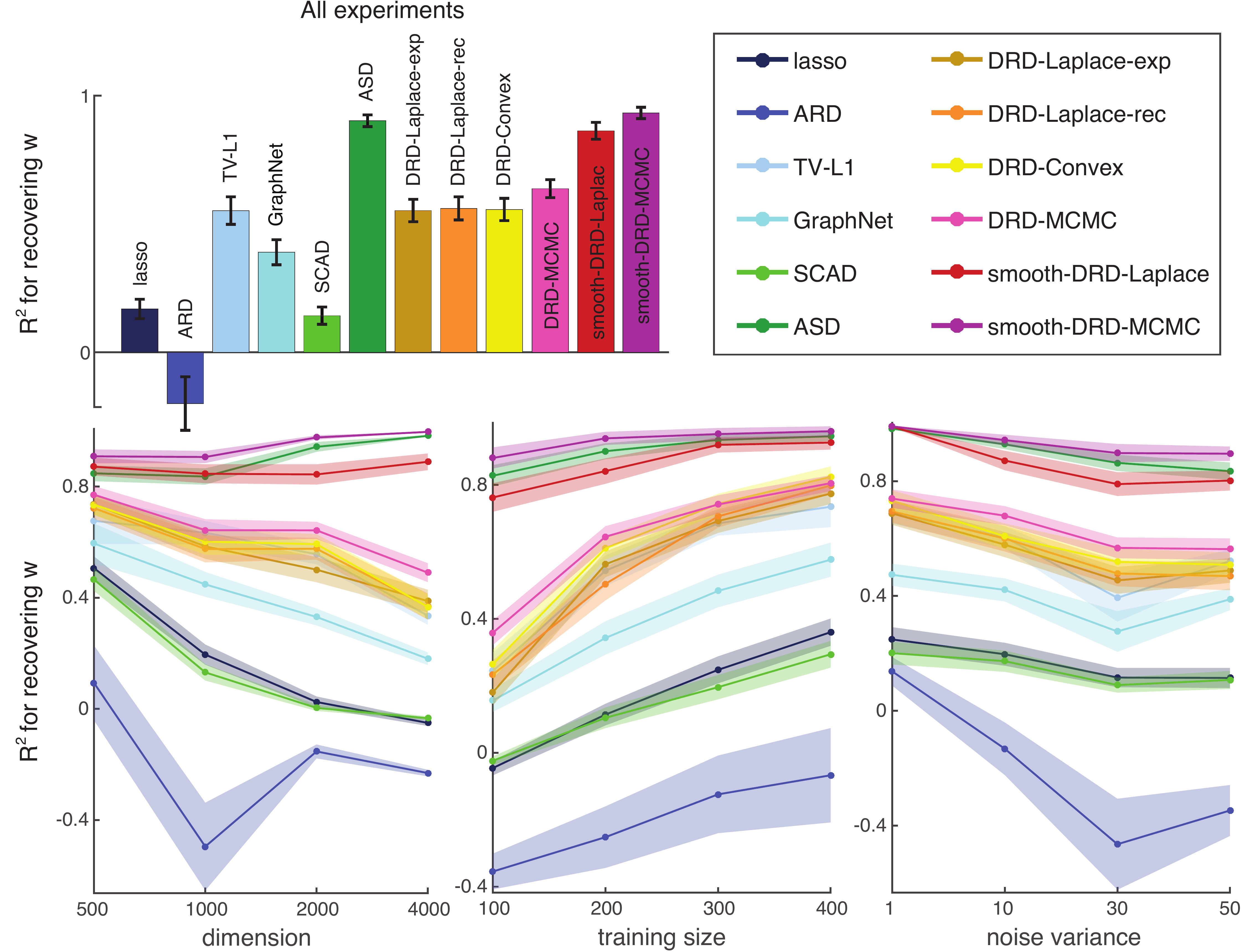} 
 \caption{Comparison of performance recovering true regression weights
 $\vw$ in simulated experiments as a function of dimensions $p$
 (lower left), number of samples $n$ (lower middle), and noise
 variance $\nsevar$ (lower right). Experiments were repeated five
 times for each of 64 combinatorial settings of four values for $p$,
 $n$, and $\nsevar$. Traces show average $R^2$ ($\pm1$ standard error of the mean (SEM)) as a
 function of each variable, and the bar plot (top row) shows average
 $R^2$ ($\pm1$ SEM) over all $5 \times 64 = 320$ experiments.}
 \label{fig:1dcurve_w}
\end{figure}

\begin{figure}[!t] 
 \centering
 \includegraphics[width=0.9\textwidth]{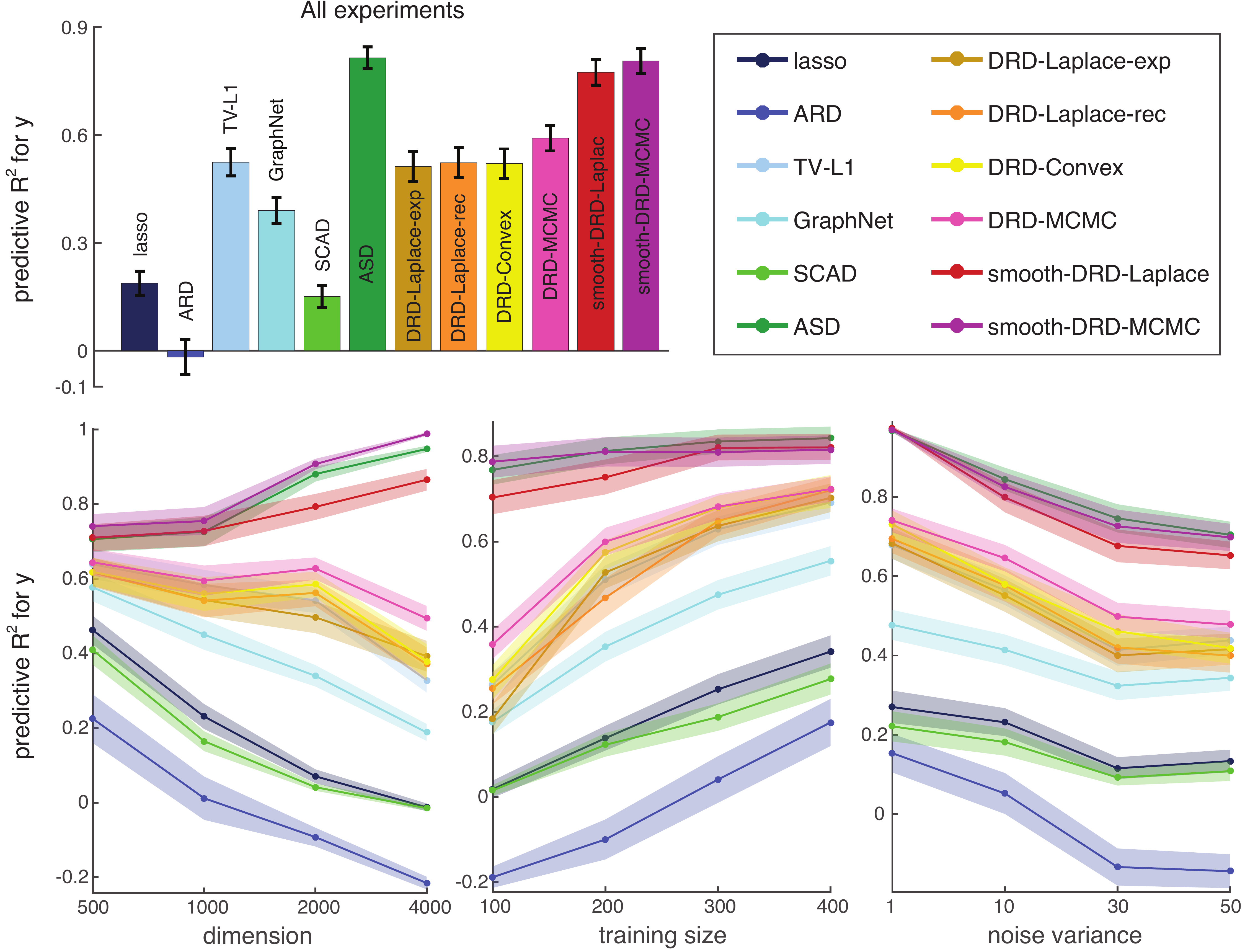} 
 \caption{Comparison of performance predicting held-out responses
 $\vy_{test}$ in simulated experiments as a function of
 dimensions $p$ (lower left), number of samples $n$ (lower
 middle), and noise variance $\nsevar$ (lower right). Traces show
 average $R^2$ ($\pm1$ SEM), and the bar plot (top row) shows
 average $R^2$ ($\pm1$ SEM) over all experiments. Simulation
 experiments were the same as in Fig.~\ref{fig:1dcurve_w}.}
 \label{fig:1dcurve_te}
\end{figure}

\begin{figure}[!t] 
 \centering
 \includegraphics[width=0.9\textwidth]{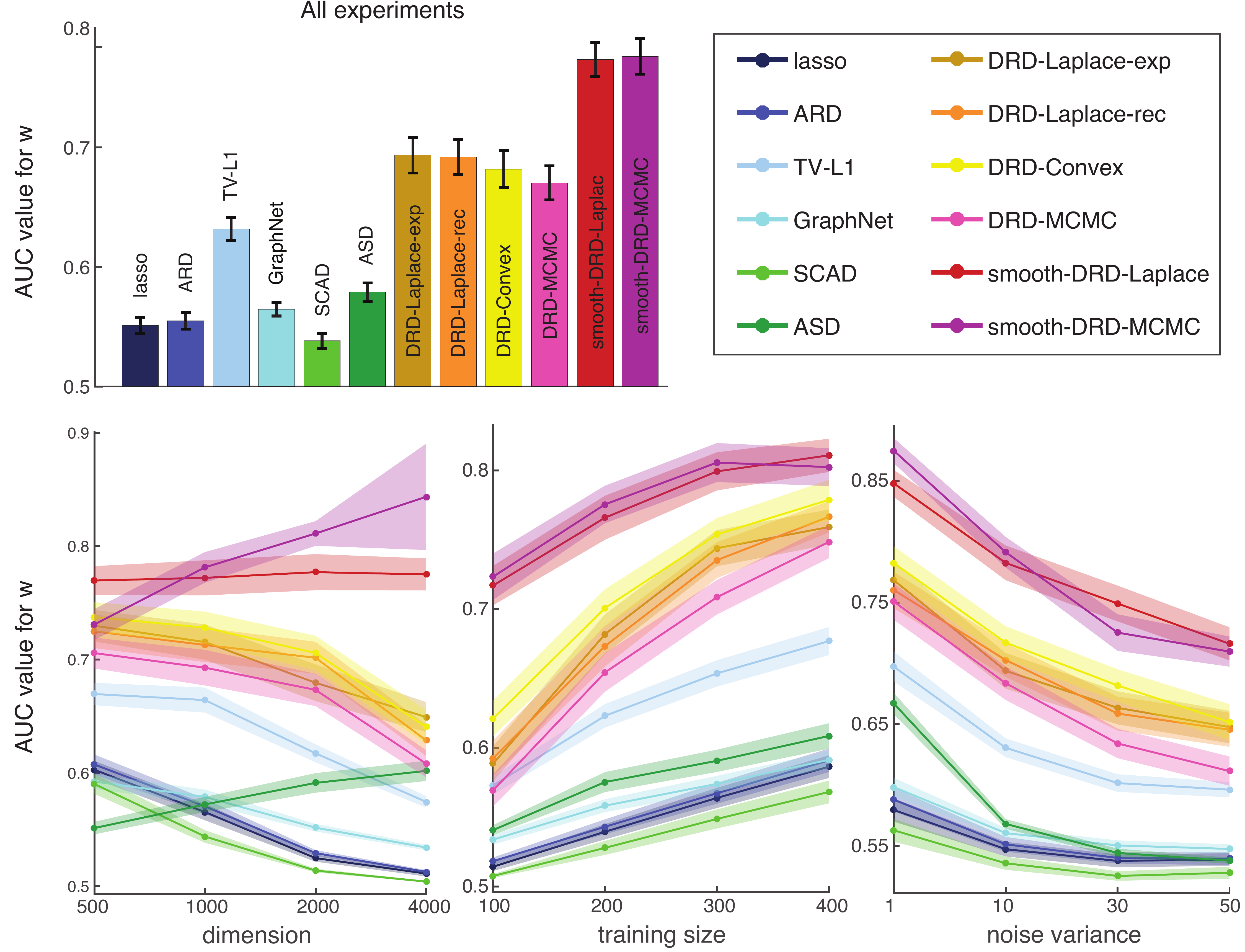} 
 \caption{Comparison of performance at recovering support of
 regression weights in simulated experiments, quantified with area
 under the ROC curve (AUC), as a function of dimensions $p$
 (lower left), number of samples $n$ (lower middle), and noise
 variance $\nsevar$ (lower right). Traces show average $AUC$ ($\pm1$
 SEM), and the bar plot (top row) shows average $AUC$ ($\pm1$ SEM)
 over all experiments. Simulation experiments were the same as in
 Fig.~\ref{fig:1dcurve_w}. Support recovery was quantified by
 taking all coefficients $|\vw_i|>0.005$ as non-zero.}
 \label{fig:1dcurve_auc}
\end{figure}
 
To quantitatively validate that our estimators succeed at identifying structured
sparse and smooth structure, we performed simulated experiments using
data drawn from the DRD generative model. For each experiment, we
generated simulated data with $n=500$ samples from a $p$-element
weight vector, and varied $p$ from 500 to 4000. We used
hyperparameters GP mean $b=-8$, GP length scale $l=p/40$, GP marginal
variance $\rho=36$, smoothness length scale $\delta=p/20$, and varied
measurement noise variance $\nsevar$ between 1 and 50. The sparsity
ratio for the sampled weights $\vw$ was approximately 0.20, where we considered weights with $|w_i| >0.005$ to be non-zero. 
We varied training set size from
$n=100$ to $400$ and kept a fixed test size of 100 samples. (We
noted that even with $n=400$ samples, the problem resides in the
$n < p$ small-sample regime). We repeated each experiment 5 times.

We compared performance of DRD estimators to the above-mentioned
estimators. Fig.~\ref{fig:1dcurve_w} shows the
reconstruction performance ($R^2$) of the true regression weights
$\vw$ for different estimators as a function of noise variance,
training set size and dimension. 

We found that Laplace and MCMC estimates for the
smooth-DRD model outperformed other estimators and were approximately
equally accurate, indicating that use of Laplace approximation did not
noticeably harm performance relative to the fully Bayesian
estimate. ASD had a good performance indicating that for these
extremely smooth weights, smoothness was a more useful form of
regularization than structured sparsity conferred by DRD. DRD models
came next. DRD-MCMC was slightly better due to the robustness of the
fully Bayesian inference. DRD-Laplace-exp and DRD-Convex employed the
exponential nonlinearity when transforming $\vu$ to the diagonal of
the covariance matrix. DRD-Laplace-rec used a soft-rectifier
nonlinearity which was more numerically stable. They had similar $R^2$
values with TV-L1 when recovering $\vw$, but were better than all the other
methods. We can also investigate the influence of each variable,
i.e. noise variance, training size or dimension. When increasing the
noise variance, all the $R^2$ values dropped; smooth-DRD-MCMC
outperformed others with $\sigma^2=50$ indicating the power of the
fully Bayesian estimate and the nontrivial effect of simultaneously
imposing local sparsity and smoothness. When increasing the training
size, DRD-Laplace models and DRD-Convex outperformed TV-L1 and were
comparable with DRD-MCMC, which was due to the decreasing optimizing
complexity. Also surprisingly, smooth-DRD estimators achieved nearly
perfect reconstruction performance over all the training sizes and all
the dimensions.

Fig.~\ref{fig:1dcurve_te} shows the $R^2$ performance for regression prediction on the test set for different estimators as a function of noise variance, training set size and dimension. The reconstruction performance for recovering the true $\vy_{test}$ is given by
$R^2 = 1-\frac{||\vy_{test}-\hat{\vy}_{test}||_2^2}{||\vy_{test}-\bar{\vy}_{test}||_2^2}$, where $\bar{\vy}_{test}= \frac{1}{n_{test}} \sum_{i=1}^{n_{test}} {\vy}_{test,i}$ is the mean of vector ${\vy}_{test}$. The top-left subfigure presents the averaged $R^2$ values and the confidence intervals for $\hat{\vy}_{test}$ over all runs. Similar to $R^2$ for $\vw$, ASD estimate, Laplace and MCMC estimates for the smooth-DRD model outperformed other estimators. 

Fig.~\ref{fig:1dcurve_auc} shows the AUC (Area Under the
receiver operator characteristic Curve) values for different estimators as a function
of noise variance, training set size and dimension. The AUC metric
quantifies accuracy in recovering the binary support for $\vw$, which
is useful for assessing the effects of structured sparsity. For this
metric, the smooth DRD estimators outperformed other
methods, and the ASD estimator performed much worse due to its lack of
sparsity. The Laplace approximation based DRD models performed
slightly better than DRD-MCMC because the sparsity of MCMC estimates
was diluted by averaging across multiple samples.




Overall, smooth-DRD outperformed all other methods using all metrics. This shows that
combining sparsity and smoothness can provide major advantages over
methods that impose only one or the other. This flexible framework
for integrating structured sparsity and smoothness is one of the
primary contributions of our work, in contrast to previous methods in
the structured sparsity literature which consider only sparsity. The code and simulated results are available online\footnote{\url{https://github.com/waq1129/DRD.git}}.

\subsection{Computational complexity and optimization}\label{mc}

\begin{figure}[!t] 
 \centering
 \includegraphics[width=0.9\textwidth]{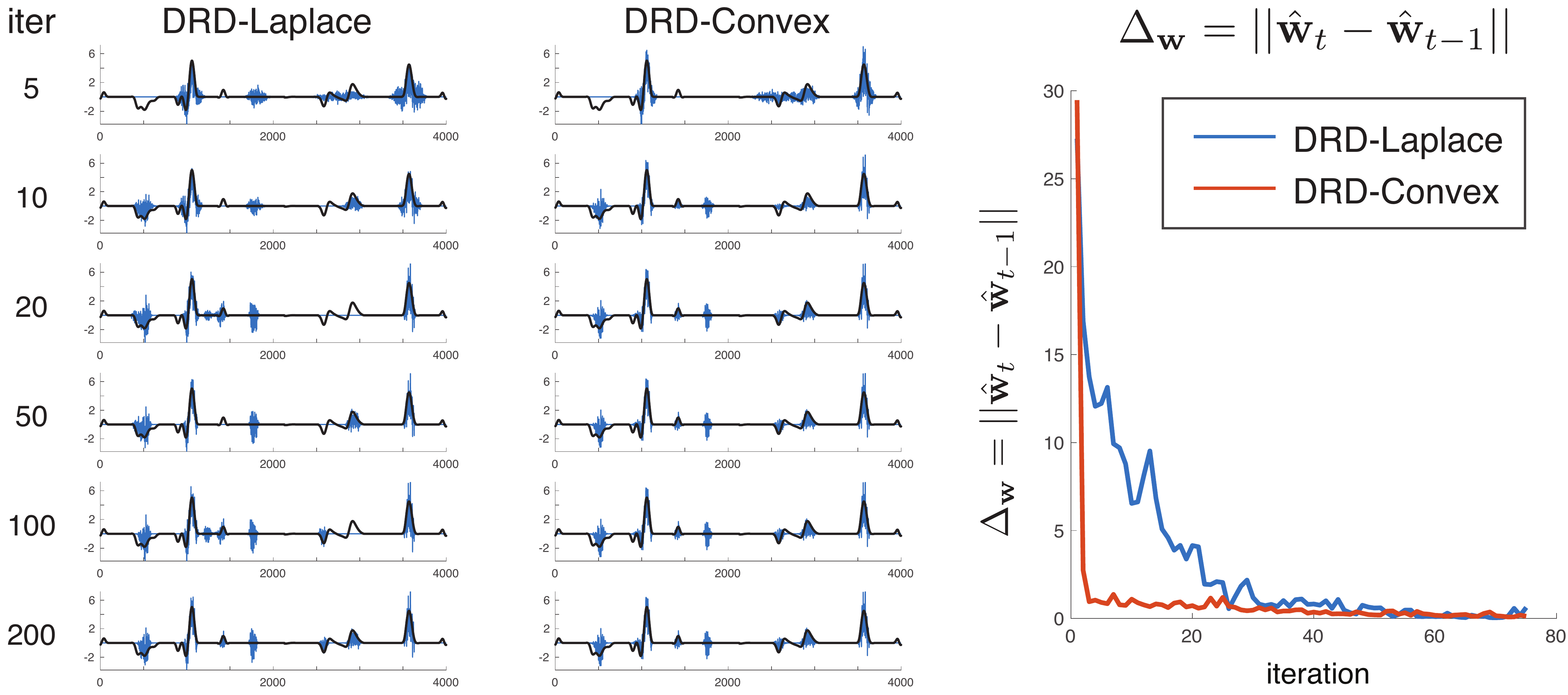} 
 \caption{Comparison of the optimization of weights $\hat{\vw}$ between DRD-Laplace and DRD-Convex. The first two columns show the weights obtained after 5, 10, 20, 50, 100, and 200 iterations with the same initialization under the two estimators, with true weights indicated in black. The third column shows the change in weights after each iteration of
the standard and convex optimization algorithms over the first 80 iterations, showing that the
convex algorithm made much smaller
adjustments to the weights after the first few iterations and thus converged more rapidly.}
 \label{fig:conv}
\end{figure}

\begin{figure}[!t] 
 \centering
 \includegraphics[width=0.9\textwidth]{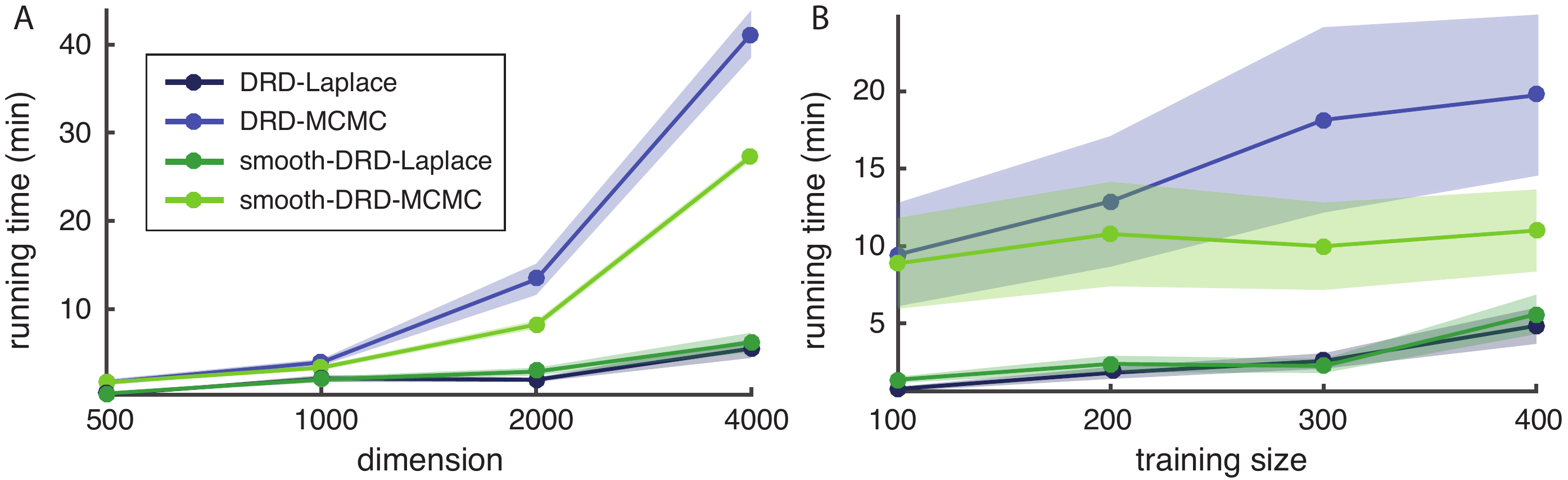} 
 \caption{Running times for DRD estimators as a function of dimensions
 $p$ (left) and number of samples $n$ (right). Each point is an
 average across 20 simulated experiments, and the shaded area
 represents $\pm$1 SEM. For MCMC, running
 time was determined by the time to collect 100 posterior samples
 after burn-in.}
 \label{fig:rt}
\end{figure}

We have described two basic approaches to inference for DRD: evidence
optimization using the Laplace approximation and MCMC sampling-based
inference. The main computational difficulty associated with
Laplace-based methods is the Hessian matrix, which provides the
precision matrix for the approximate Gaussian posterior
distribution. This matrix costs $O(p^2)$ to store and contributes
$O(p^3)$ time complexity for computation of the log-determinant. We
can reduce these costs to $O(p_{\scriptscriptstyle f}^2)$ storage and
$O(p_{\scriptscriptstyle f}^3)$ time, where
$p_{\scriptscriptstyle f}<p$ is the number of non-zero Fourier
coefficients in the spectral domain representation of the latent
Gaussian process. This savings can be significant in problems with
strong region sparsity, that is, when the zero coefficients arise in
large contiguous blocks. However, in very high dimensions, the
Laplace based methods may be practically infeasible due to the
impossibility of storing the Hessian.

We also described a two-step convex relaxation of the Laplace method
(DRD-Convex), which takes more time per iteration than the standard
Laplace method (DRD-Laplace) due to the need for a two-step
optimization procedure. However, we find that the DRD-Convex takes
fewer iterations to converge (see Fig. \ref{fig:conv}), and in some
cases proves more successful at avoiding sub-optimal local optima.

The MCMC-based inference has a time complexity of only
$O(n^2p_{\scriptscriptstyle f})$ per sample, due to the fact that
there is no need to compute the Hessian. However, MCMC-based inference
is typically slower due to the need for a burn-in period and the
generation of many samples from the posterior.
 Fig.~\ref{fig:rt}
shows a comparison of running time for the two inference
methods for both DRD and smooth-DRD models.
For the Laplace method, we used a stopping criterion that the change
in $\vw$ was less than 0.0001. For the MCMC method, we assessed burn-in using a criterion on the relative change in $\vw$, and then
collected 100 posterior samples. Inference for the smooth-DRD model
was faster than for standard DRD due to the fact that the smoothing
prior effectively prunes high frequencies, reducing the dimensionality
of the search space for $\vw$. In these experiments, increasing
dimension $p$ elicited larger increases in computation time than
increasing training set size $n$.

\section{Phase transition in sparse signal recovery}

\begin{figure}[!t] 
 \centering
 \includegraphics[width=1\textwidth]{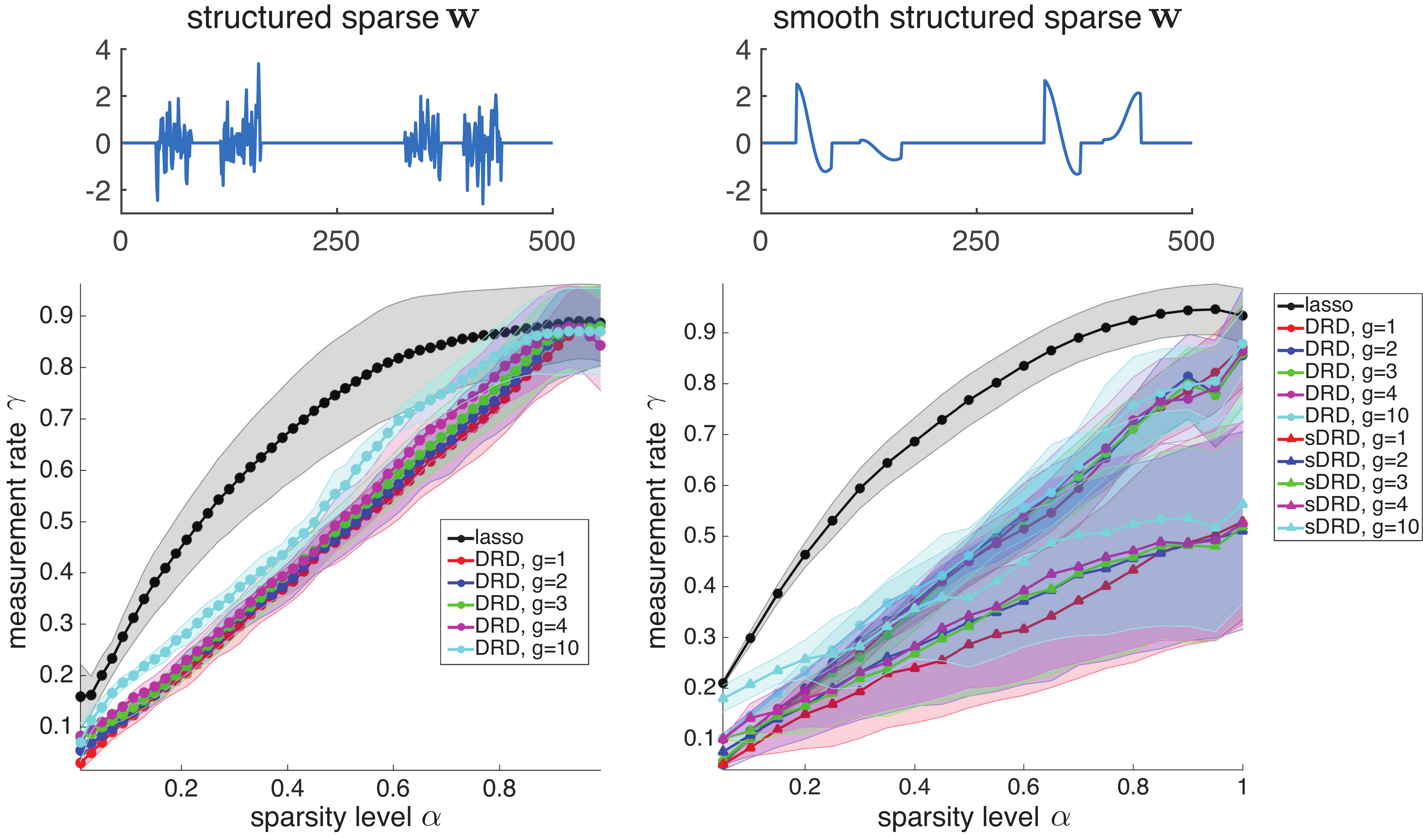} 
 \caption{
{\bf Phase transitions for DRD and smooth-DRD (sDRD) estimators on signals with structured sparsity.}
 Top row shows example signals $\vw$ of dimensions $p$ = 500, which
 contain randomly positioned blocks of non-zero
 coefficients. Non-zero coefficients were clustered into varying
 numbers of groups $g$, and drawn either iid from a
 standard normal distribution (left column) or from a Gaussian with a smoothing kernel (length scale was 20) (right column), to illustrate the effects of smoothness. To
 compute phase transition curves, we analyzed the recovery behavior of
 each estimator at every point in a 2D grid of sparsity levels $\alpha$ and
 measurement rates $\gamma$. At each point, we generated 10 random
 signals $\vw$, projected them noiselessly onto a
 random Gaussian measurement matrix $X$, and computed lasso and DRD
 estimates $\hat \vw$. We then calculated the $R^2$ value of the
 estimates for each trial at every grid point $(\alpha,\gamma)$.
 An estimator was considered to achieve perfect recovery if all 10
 trials resulted in $R^2>0.95$, and perfect failure if all 10 trials
 resulted in $R^2\leq 0.95$; remaining points were considered to fall
 in the phase transition region. For each estimator, the shaded
 region indicates the phase transition region, and solid line
 indicates its center of mass along the y-axis.}
 \label{pt}
\end{figure}

Compressive sensing focuses on the recovery of sparse high-dimensional
signals in settings where the number of signal coefficients $p$
exceeds the number of measurements
$n$. 
Recent work has shown that the recovery of sparse signals exhibits a
phase transition between perfect and imperfect recovery as a function
of the number of measurements
\citep{ganguli2010statistical,amelunxen2014living}. Namely, when the
measurement fraction $\gamma=n/p$ exceeds some critical value
that depends on signal sparsity, the signal can be recovered perfectly
with probability approaching 1, whereas for $\gamma$ below this value,
estimates contain errors with probability approaching 1. However,
these results were derived for the case where non-zero coefficients
are randomly located within the signal vector. Here we show that DRD
can obtain dramatic improvements over the phase transition curve for
{\it iid}\ sparse signals when the non-zero coefficients arise in
clusters.

We performed simulated experiments to examine the effects of group
structure on the empirical phase transition between perfect and
imperfect recovery of sparse signals. The measurement equation is
given by the noiseless version of the linear system we have considered
so far: $\vy = X\vw$, where $\vw \in \RR^p$ is the sparse signal,
$\vy\in \RR^n$ is the (dense) measurement vector, and here
$X \in \RR^{n \times p}$ is a (short, fat) random measurement matrix
with entries drawn {\it iid}\ from a standard normal distribution.
We define the sparsity of the signal as $\alpha = k/p$, where $k$ is
the number of non-zero signal coefficients in $\vw$. 

To explore the effects of group structure, we considered the signal
coefficients in $\vw$ to have 1D spatial structure and introduced a
parameter $g$ specifying the number of spatial groups or clusters into
which the non-zero coefficients were divided. When $g=1$, the non-zero
coefficients formed a single contiguous block of length $k$, with
location uniformly distributed within $\vw$. When $g>2$, the non-zero
coefficients were divided into $g$ blocks of size $k/g$, and the
locations of these blocks were uniformly distributed within $\vw$
subject to the constraint that blocks remained disjoint. Once the
sparsity pattern was determined, we sampled the non-zero coefficients
from a standard normal distribution.


Fig. \ref{pt} shows the empirical phase transition curves for lasso
and DRD estimators for sparse signals with non-zero coefficients
clustered into varying numbers of groups $g$. These curves show the
boundary between perfect and imperfect signal recovery for different
estimators in the 2D space of signal sparsity level $\alpha$ and
measurement fraction $\gamma$. The left bottom plot shows that DRD
estimators achieved perfect signal recovery for much lower measurement
rates, even when non-zero coefficients were clustered into as many as
10 groups. Here DRD achieved transition to perfect recovery along the
main diagonal, whereas lasso exhibited an arc-shape transition curve
described previously
\citep{ganguli2010statistical,amelunxen2014living}, indicating that
more measurements were required to recovery signals of equal
sparsity. In the right bottom plot, we generated the non-zero coefficients
from a Gaussian distribution with a smoothing kernel whose length scale equaled to 20, so that
non-zero coefficients were smooth as well as sparse. In these plots,
we compared lasso estimates (which do not benefit from group or smooth
structure) to standard DRD and smooth-DRD estimates. This reveals
that smoothness allows for further reductions in measurement rates,
with perfect signal recovery achievable well below that of the
non-smooth DRD estimates.



\section{Applications to brain imaging data}

Functional magnetic resonance imaging (fMRI) measures blood
oxygenation levels, which provide a proxy for neural activity in different
parts of the brain. Although these measurements are noisy and
indirect, fMRI is one of the primary non-invasive methods for
measuring activity in human brains,
and 
it has provided insight into the neural basis for a wide variety of
cognitive abilities and functional pathologies.

A primary problem of interest in the fMRI literature is ``decoding'',
which involves the use of linear classification and regression methods
to identify the stimulus or behavior associated with measured brain
activity. Decoding is a challenging statistical problem because the
number of volumetric pixels or ``voxels'' measured with fMRI is
typically far greater than the number of trials in an experiment; a
full brain volume typically contains 50K voxels whereas most
experiments produce only a few hundred observations. 

Standard approaches to decoding have therefore tended to exploit
sparsity, corresponding to the assumption that only a small set of
brain voxels are relevant for decoding a particular set of stimuli
\citep{Carroll09}.
However, the set of voxels useful for a specific
decoding task are not randomly distributed throughout the brain, but
tend to arise in clusters; if one voxel carries information useful for
decoding, it is {\it a priori} likely that nearby voxels do too, given
that voxels represent an arbitrary discretization of continuous
underlying brain structures.
We therefore explored brain decoding as an ideal application for
evaluating the performance of our estimators. 


\subsection{Gambling task}

\begin{figure}[!t] 
 \centering
 \includegraphics[width=1\textwidth]{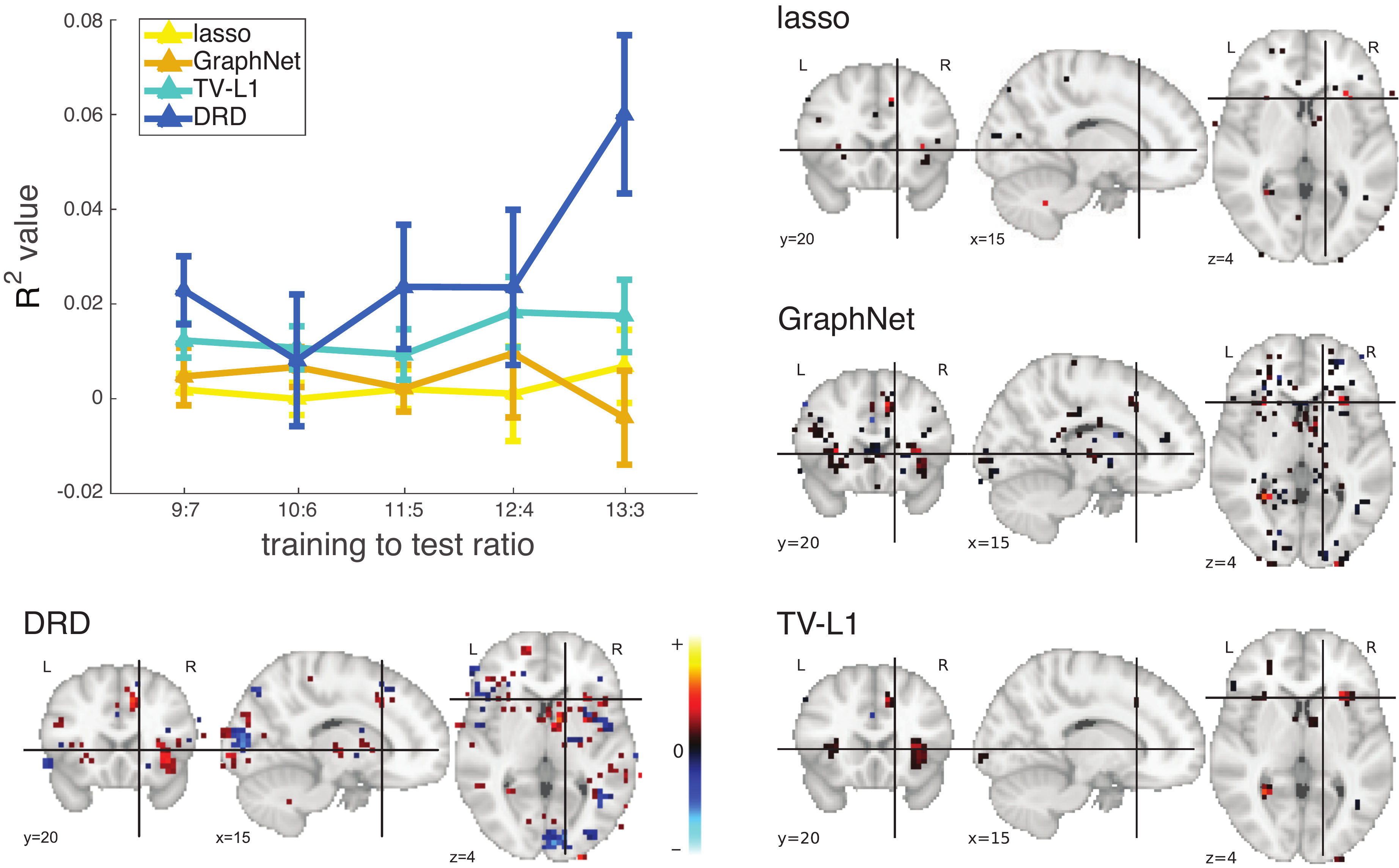} 
 \caption{Top left panel shows average test $R^2$ values on the
 gambling dataset as a function of the train-test ratio for lasso,
 GraphNet, TV-L1 and DRD. The x-axis is the train-test split ratio
 and the y-axis is the $R^2$ criterion. The remaining panels show the estimated fMRI weight maps, overlaid on a structural fMRI image. Colors indicate the
 sign and magnitude of the weights (see color bar, red for positive, blue for negative, black for small). The DRD figure
 was obtained by cutting off small weight coefficients with a small
 threshold value at 0.004 (about 12\% of the maximal absolute
 coefficient value).}
 \label{gamble}
\end{figure}
We first considered the regression problem of decoding gains and
losses from fMRI measurements recorded in a gambling task
\citep{tom2007neural,tom2011gamble}. In this experiment, event-related
fMRI was administered while healthy human participants performed a
decision-making task. In each trial, a gamble with a potential gain
and loss (each with 50\% probability) was presented for 3s, and the
participants were instructed to decide whether to accept or reject the
gamble. Experimenters varied amount of the potential gain and loss
across trials. The regression task is to predict the gain of the gamble from the fMRI images recorded during the decision-making task.

After standard preprocessing, the regression dataset consisted of 16
subjects with 48 fMRI measurements per subject (resulting from 6
repeated presentations of 8 different gambles). fMRI measurements were
obtained from a 3D brain volume of $40\times 48\times 34$ voxels, each
of size $4\times 4\times 4$ mm, from which a subset of 33,177 valid
brain voxels were used for analysis. The full dataset of 16 subjects
therefore contained $n=768$ samples in a $p=33,177$ dimensional space.

We evaluated inter-subject prediction performance by estimating
regression weights with data from a subset of the 16 subjects, and
computing prediction accuracy for data from held-out subjects. To
assess the performance, we varied the train-test ratio in number of
subjects from 9:7 to 13:3. We performed 10 different random train-test
splits for each ratio. We used 5-fold cross-validation to set
hyperparameters for all models, including DRD. For DRD models, the
Laplace approximation was intractable due to the high dimensionality
of the weight vector ($p=33,177$). We therefore computed MAP estimates
of the latent vector $\vu$ conditioned on the hyperparameters, and set
hyperparameters using cross-validation.

The curves in the top left panel of Fig.~\ref{gamble} show the
performance of lasso, GraphNet, TV-L1 and DRD estimators. We found
that DRD outperformed other estimators at nearly all train-test
ratios, with a noticeable advantage at the largest training set
size. However, we noted that the SNR of this dataset was low, making
inter-subject prediction difficult and resulting in low accuracy for
all methods. A non-trivial preprocessing stage, such as
hyper-alignment \citep{chen2015reduced}, could be used to map
different subjects into a shared subspace, which reduces the low SNR induced by inter-subject variability and could possibly improve
performance.

Fig.~\ref{gamble} also shows the inferred regression weights for each
estimator. The GraphNet and lasso weights had high sparsity,
presumably due to the low SNR of the dataset, while TV-L1 weights
exhibited small blocks of non-zero coefficients with constant value
within each block, consistent with the structure expected for the
TV-L1 penalty. The DRD weights were not sparse in a strict L0 sense,
due to the fact that sparsity arises from soft-rectification of
negative latents governing the prior variance. We therefore
thresholded DRD weights for plotting purposes, revealing that the
weights contributing most to prediction performance tended to cluster,
as expected, although weights within each cluster were not
constant. One noteworthy observation is that DRD estimate had positive
(red) as well as negative weights (blue), while other estimates were
largely devoid of regions with negative weights. Note that voxels in
black indicate weights close to zero, which therefore contributed relatively
little to readout.


\subsection{Age prediction task}

\begin{figure}[!t] 
 \centering
 \includegraphics[width=1\textwidth]{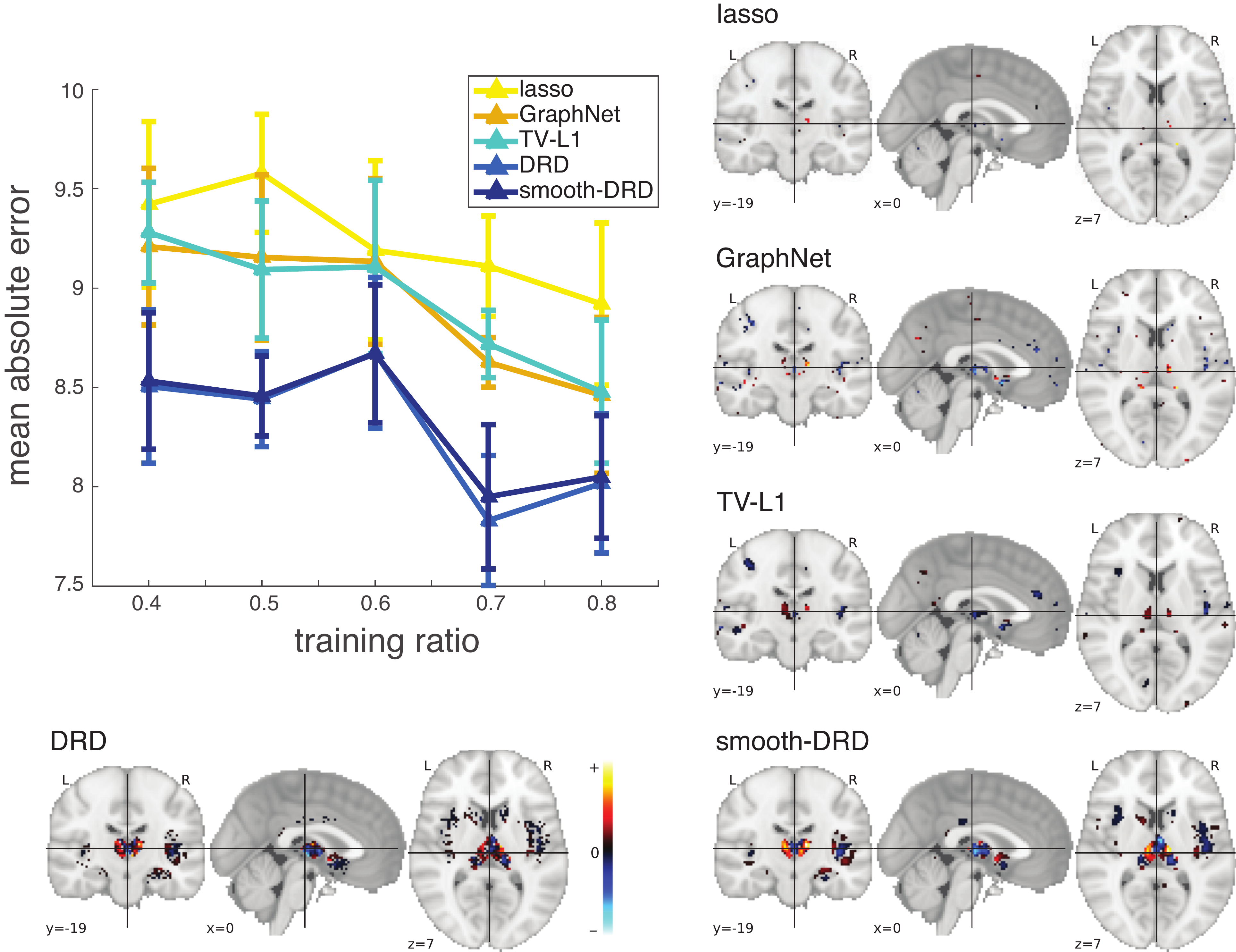} 
\caption{Top left panel shows average test mean absolute error values on the age regression dataset as a
 function of the training ratio for lasso, GraphNet, TV-L1, DRD
 and smooth-DRD. The x-axis is the ratio of the training data over
 the entire dataset and the y-axis is the mean absolute error criterion. The remaining panels show the estimated MRI weight maps, overlaid on a structural MRI image. Colors indicate the sign and magnitude of the weights (see color bar).}
 \label{age}
\end{figure}
 
Next we considered the problem of predicting a subject's age from a
measured map of gray-matter concentration, using data from the Open
Access Series of Imaging Studies (OASIS) \citep{marcus2007open}. The
OASIS dataset consisted of T1-weighted MRI scans data from 403
subjects aged 18 to 96, with 3 or 4 scans per subject. One hundred of
these subjects were over 60 years of age and had been clinically
diagnosed with Alzheimer's. The repeated scans provided high
signal-to-noise ratio, making the dataset feasible for inter-subject
analyses.

A natural regression problem for this dataset is to predict the age of
subject from their anatomical MRI data. The full dataset consisted of 403
samples with a $91\times 109\times 91$ 3D volume and 129,081
valid voxels. To assess the performance, we varied the training ratio
from $0.4$ to $0.8$ out of the 403 subjects, and averaged over 5
random splits for each ratio.

The curves in the top left panel of Fig.~\ref{age} show mean absolute errors between the true age and the
predicted age for each estimator, evaluated on test data. The DRD and
smooth-DRD estimators, which performed similarly well, achieved lower
error than lasso, GraphNet, and TV-L1 estimators. The inferred
regression weights (Fig.~\ref{age}) reveal that the most informative
voxels for predicting age lie in thalamus and the basal ganglia; this
is consistent with previous findings about the relationship between
Alzheimer's disease progression and anatomical changes in gray matter
\citep{de2008strongly, cho2014shape}. Thalamus and basal ganglia were
more clearly highlighted in the inferred DRD and smooth-DRD regression
weights, which contained larger and more concentrated regions around
these two structures. The smooth-DRD weights exhibited stronger
spatial clustering than DRD weights, although regression performance
was not noticeably different between the two estimators.

\subsection{Visual recognition task}

\begin{figure}[!t] 
 \centering
 \includegraphics[width=0.6\textwidth]{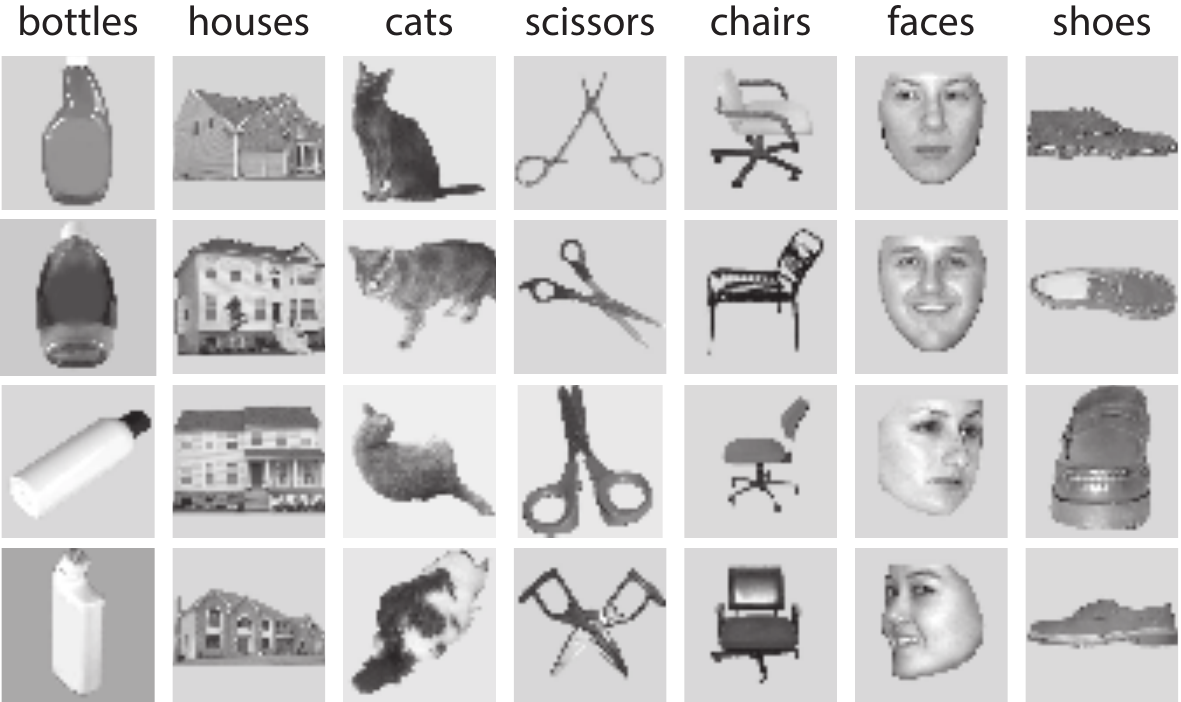} 
 \caption{Examples of the stimuli for 7 categories (except for scrambled control images).}
 \label{examples}
\end{figure}

In a third application, we examined the problem of decoding faces and
objects from fMRI measurements during a visual recognition task. We
used a popular fMRI dataset from a study on face and object
representation in human ventral temporal cortex
\citep{haxby2001distributed}. In this experiment, 6 subjects were
asked to recognize 8 different types of objects (bottles, houses,
cats, scissors, chairs, faces, shoes and scrambled control images,
examples in Fig. \ref{examples}). Each subject
participated 12 sessions of experiment. In each session, the subjects viewed images of eight object categories, with 9 full-brain measurements per category.

We assessed performance by training linear classifiers to discriminate
between pairs of objects, e.g., face vs.\ bottle, for each of the $28$
possible binary classifications among 8 objects. We trained the
weights $\vw$ for each model to linearly map fMRI measurements $\vx$
to binary labels $y \in \{-1,+1\}$ by minimizing squared error. Note
that a Bernoulli log-likelihood or logistic loss would be more
appropriate for this binary classification task, but we used squared
error loss because it allows for analytic marginalization over
weights. We assessed decoding accuracy on test data using predicted
labels $\hat y = \textrm{sign}(\vw\trp \vx)$. We divided 12 sessions
of data per subject into train-test splits of 5:7, 6:6 and 7:5. When
training with data from $N$ sessions, the training dataset consisted
of 18N full-brain measurements (9 measurements per category $\times$ 2
categories). Each measurement contained 24,083 valid voxels from a
$40\times 64\times 64$ 3D volume.

Fig. \ref{classification} shows the classification performance of
lasso, GraphNet, TV-L1, DRD, and smooth-DRD estimators, averaged over
6 subjects and across the three train-test splits. The smooth-DRD
estimate achieved the highest accuracy for most of the binary
classifications, while the DRD estimate achieved second highest
accuracy. The left column in Fig. \ref{haxby_brain} shows the
regression weights estimated for the house vs.\ bottle classification
task. The DRD and smooth-DRD weights both had significant positive
regions in the parahippocampal place area (PPA), an area known to
respond to images of places \citep{epstein1999parahippocampal} and
negative regions in the lateral occipital complex (LOC), an area known
to respond to objects \citep{eger2008fmri}. By comparison, TV-L1 and GraphNet weights in LOC were neither strong
nor clustered. The right column in Fig. \ref{haxby_brain} shows
regression weights for the face vs.\ scrambled-pixels classification
task. All methods managed to discover active responses around LOC and
fusiform face area (FFA) (specialized for facial recognition)
\citep{kanwisher1997fusiform}, though DRD and smooth-DRD weights
exhibited fewer isolated non-zero weights in areas far from the
temporal and occipital lobes.

\begin{figure}[!t] 
 \centering
 \includegraphics[width=1\textwidth]{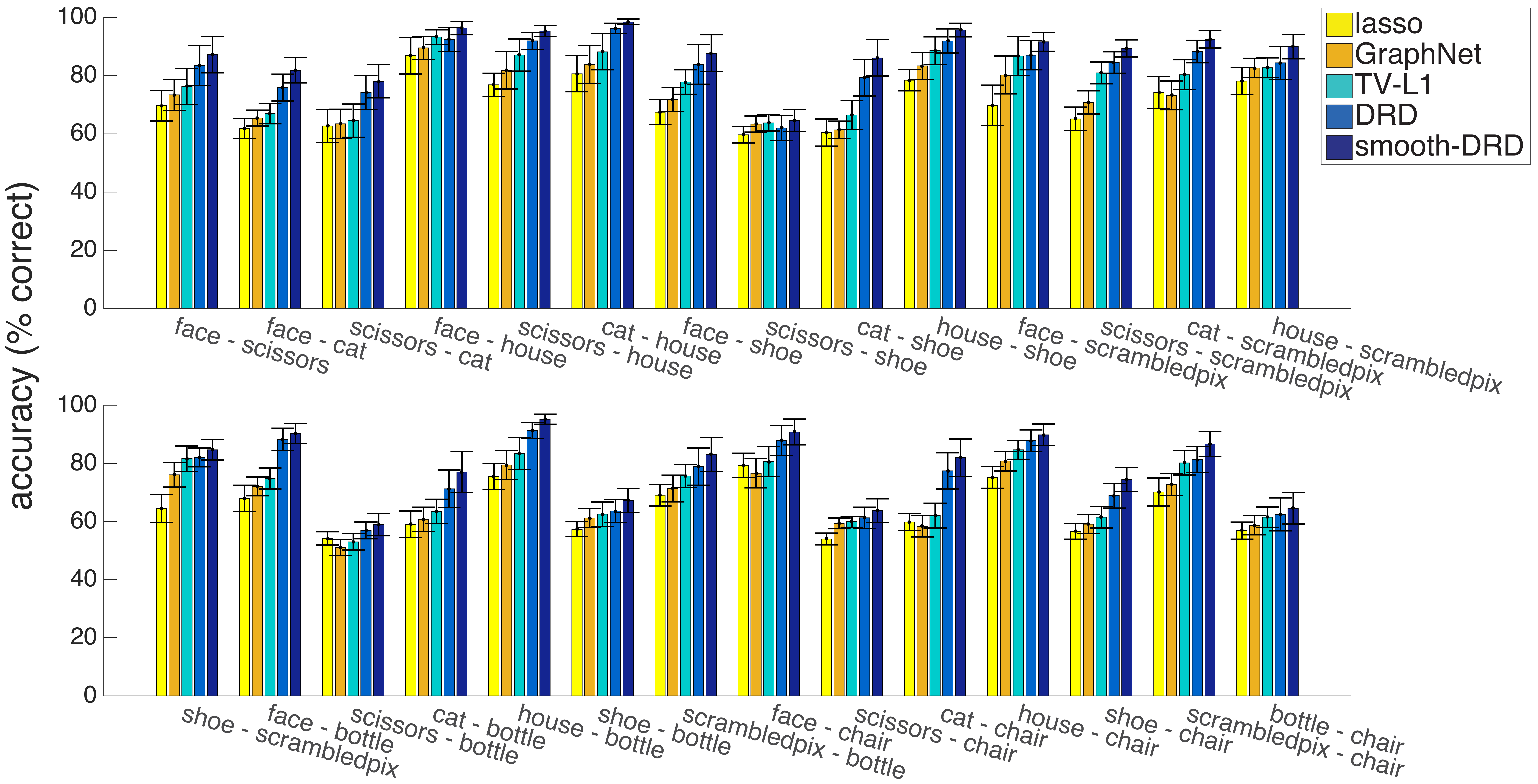} 
 \caption{Classification accuracy performance for lasso, GraphNet,
 TV-L1, DRD and smooth-DRD, averaged across 6 subjects and three
 different train-test splits (5:7, 6:6 and 7:5). Error bars
 indicate $\pm$1 SEM, averaged over
 train-test splits. The x-axis labels indicate pairs of object
 categories considered for binary classification. }
 \label{classification}
\end{figure}

%

\begin{figure}[!t] 
 \centering
 \includegraphics[width=1\textwidth]{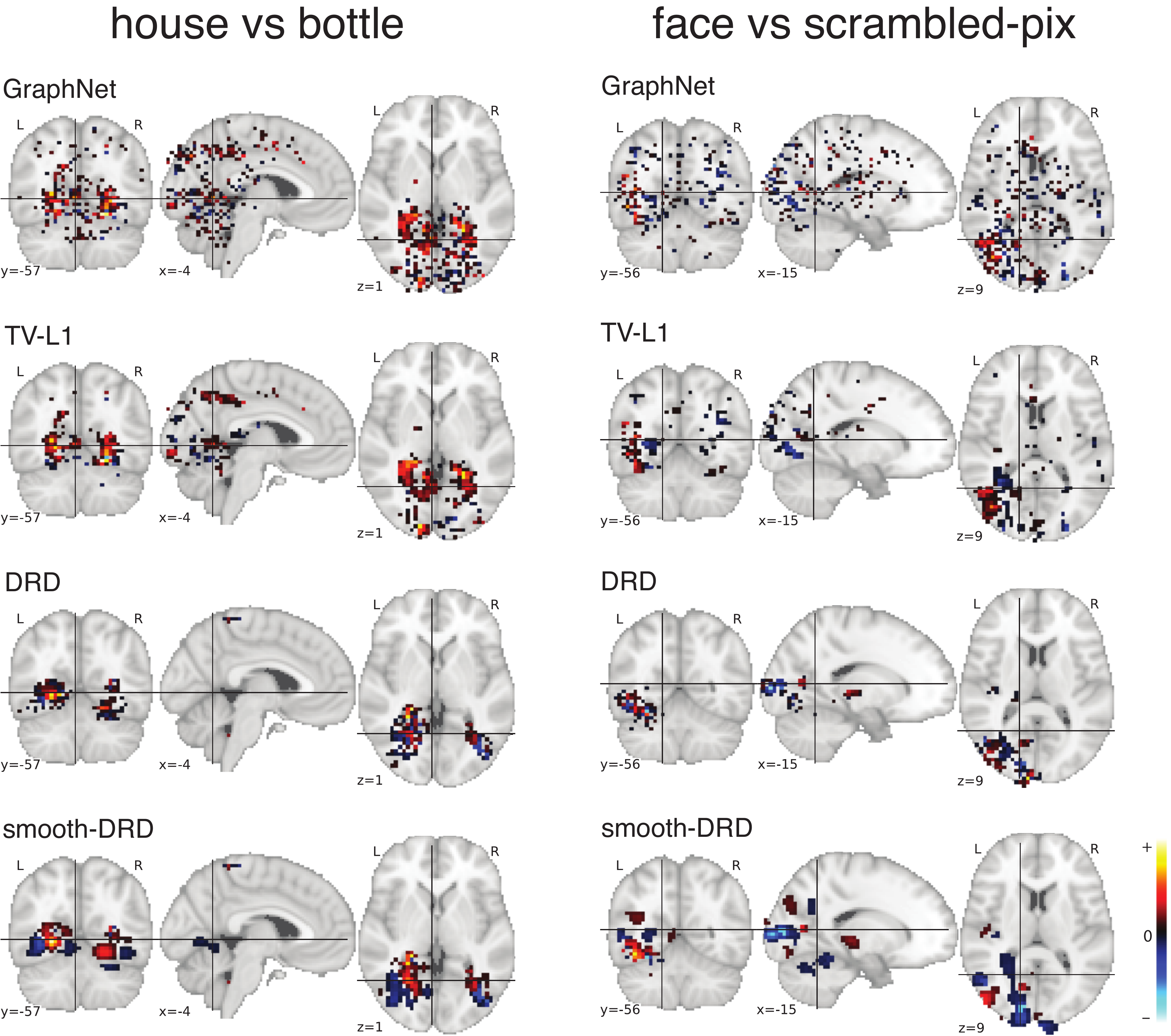} 
 \caption{\textbf{Left column}: weight maps for the house vs bottle pair. \textbf{Right column}: weight maps for the face vs scrambled-pix pair. The methods shown are GraphNet, TV-L1, DRD and smooth-DRD.}
 \label{haxby_brain}
\end{figure}

\section{Discussion}

In this paper, we introduced dependent relevance determination (DRD),
a hierarchical Bayesian model for sparse, localized, and smooth
regression weights. This model is appropriate for regression settings
in which the regressors can be arranged geometrically as a vector,
matrix, or tensor, and exhibit local dependencies within this
structure (e.g., tensors of 3D brain measurements).

The DRD model takes its inspiration from the automatic relevance
determination (ARD) model \citep{tipping2001sparse}, but adds a
Gaussian process to introduce dependencies between prior variances of
regression weights as a function of distance between their regressors.
Samples from the DRD prior therefore exhibit clustering of
non-zero weights. However, weights sampled from the standard DRD
prior are uncorrelated, meaning there is no tendency for such weight tensors
to be smooth. For this reason, we introduced the
smooth-DRD model, which composes the standard DRD prior with a smoothness-inducing covariance
function \citep{sahani2003evidence,park2011receptive}. Weights sampled
from the resulting model tend to be sparse as well as smooth, with islands
of smooth, non-zero weight features surrounded by oceans of zeros. 

We described two methods for inferring the model parameters: one based
on the Laplace approximation and a second based on MCMC. We proposed
a novel variant of the Laplace based approach involving a two-stage
convex relaxation of the log posterior. 

Lastly, we carried out simulated and real application experiments to compare DRD with a variety of
other methods, including lasso, SCAD, GraphNet, TV-L1, and etc. For an $L_2$ loss, a convex $L_1$ penalty leads to strong amplitude bias for lasso, while a non-convex penalty, e.g., SCAD, can overcome the strong bias. Lasso could underfit in order to get a tight support due to incorrect amplitudes of the coefficients when doing cross-validation using a $L_2$-based criterion. Therefore, we included SCAD apart from lasso and employed AUC as the metric for support identification. The synthetic experiments generated true weight vectors from the proposed generative model, thus aiming at validating the proposed model. We also examined phase transitions between perfect and imperfect recovery using data from a generating model different from DRD, showing that the
DRD and smooth-DRD model could achieve perfect recovery with far fewer
measurements when the non-zero weights in a signal were clustered. We further applied our models to three real-world brain imaging datasets. We found that DRD and smooth DRD achieved better prediction performance than previous methods, while also achieving
high interpretability with regression/classification weights defined by smooth,
clustered sets of voxels. Note that for the final fMRI decoding application task, a Bernoulli
noise model (corresponding to a logistic loss function) would have
been more appropriate than the Gaussian noise model we assumed, due to
the binary nature of the outputs. Gaussian noise is a useful modeling
assumption for the DRD model because it yields an analytical
expression for the conditional evidence (eq.~\ref{eq:evidence2}),
which can be directly optimized for the latent process governing
region sparsity. An important direction for future research will
therefore be to extend DRD to incorporate Bernoulli and other
likelihoods. Such extensions will need to use approximate inference
methods or sampling to compute the integral over regression weights
(eq.~\ref{eq:evidence}), but there is no conceptual barrier to
incorporating region sparsity into models with binary and other
non-Gaussian outputs.

The DRD and smooth-DRD models offer a powerful statistical framework
for attacking problems in which sparsity is overlaid with local
dependencies, a scenario that arises commonly in (for example) spatial
and temporal regression problems. Recent work has shown successful
application of a closely related model for capturing dependencies
between sparse variables in genomic data
\citep{engelhardt2014Bayesian}. In future work, we expect the DRD
framework to find applications beyond the regression/classification setting,
including structured latent factor models \citep{chen2015reduced} and
false discovery rate estimation \citep{Tansey17}.





\acks{This work was supported by grants from the Sloan Foundation
 (JWP), McKnight Foundation (JWP), Simons Collaboration on the Global
 Brain (SCGB AWD1004351, JWP), and the NSF CAREER Award (IIS-1150186,
 JWP).}

\bibliography{bib_depARD}

\newpage
\appendix
\section{The Hessian of the negative log-posterior in Laplace approximation for smooth-DRD}\label{hess}
Now we derive the Hessian matrix inside of the inverse of (eq.~\ref{eq:Lam}).
\begin{equation}
 H=-\frac{\partial^2}{\partial \vu \partial\vu^\top} \Big[ \log p(\vy| X,\vu,\nsevar,\delta) + \log p(\vu| \vtheta_{drd}) \Big].
\end{equation}
The first term to take the partial derivatives of is (eq.~\ref{eq:datalike1}):
\begin{align}
 \log p(\vy|X,\vu,\nsevar,\delta) &= -\frac{1}{2}\log|XCX^\top +\nsevar I|-\frac{1}{2}\vy^\top(XCX^\top +\nsevar I)^{-1}\vy + const,
\end{align}
 and the second is (eq.~\ref{eq:datalike2}):
 \begin{align}
 \log p(\vu|\vtheta_{drd}) &= -\frac{1}{2}(\vu-b\vones)^\top K^{-1}(\vu-b\vones)-\frac{1}{2}\log |K|+const.
\end{align}
Define $S = X C X\trp + \nsevar I$, where $C = C_{drd}^{\frac{1}{2}}\, \Sigma\, C_{drd}^{\frac{1}{2}}$. 
Let $Z = X C_{drd}^{\frac{1}{2}}\, \Sigma^{\frac{1}{2}}$, then,
\begin{eqnarray}
S &=& XC_{drd}^{\frac{1}{2}}\, \Sigma\, C_{drd}^{\frac{1}{2}}X^\top+\nsevar I = ZZ^\top+\nsevar I,\\
\log p(\vu|\vy,X, \vtheta) &=& -\frac{1}{2}\vy^\top S^{-1} \vy -\frac{1}{2}\log|S|-\frac{1}{2}(\vu-b\vones)^\top K^{-1}(\vu-b\vones)+const.
 \end{eqnarray}
The first derivative with respect to $\vu_i$ is given by:
\begin{eqnarray}
\frac{\partial}{\partial \vu_i} \log p(\vu|\vy,X, \vtheta) &=& \frac{1}{2}\ \frac{\partial}{\partial \vu_i}\left(-\vy^\top S^{-1} \vy -\log|S| -(\vu-b\vones)^\top K^{-1}(\vu-b\vones)\right) \\
&=&\mbox{Tr} \left [Z^\top S^{-1}\vy\vy^\top S^{-1} \left( \frac{\partial}{\partial \vu_i}Z \right)-Z^\top S^{-1} \left( \frac{\partial}{\partial \vu_i}Z \right)\right]\\
&&-\left[K^{-1}(\vu-b\vones)\right]_i,
\end{eqnarray} 
where
\begin{eqnarray}
 \frac{\partial}{\partial \vu_i}Z &=& X \left( \frac{\partial}{\partial \vu_i} C_{drd}^{\frac{1}{2}} \right) \Sigma^{\frac{1}{2}}.
\end{eqnarray} 
The second derivative with respect to $\vu_j$ is given by:
\begin{eqnarray}
\frac{\partial^2}{\partial \vu_i \partial \vu_j } \log p(\vu|\vy, X,\vtheta) &=&H_1 + H_2 + H_3-K^{-1}_{ij}=-H_{ij}, 
\end{eqnarray} 
\begin{eqnarray}
H_1 &=& \mbox{Tr} \left [Z^\top S^{-1}\vy\vy^\top S^{-1} \left( \frac{\partial^2}{\partial \vu_i \partial \vu_j } Z \right)-Z^\top S^{-1}\left( \frac{\partial^2}{\partial \vu_i \partial \vu_j } Z \right)\right],
\end{eqnarray} 
\begin{eqnarray}
H_2 &=& \mbox{Tr} \left [\left( \frac{\partial}{\partial \vu_j}Z \right)^\top S^{-1}\vy\vy^\top S^{-1} \left( \frac{\partial}{\partial \vu_i}Z \right)-\left( \frac{\partial}{\partial \vu_j}Z \right)^\top S^{-1} \left( \frac{\partial}{\partial \vu_i}Z \right)\right],
\end{eqnarray} 
\begin{eqnarray}
H_3 &=& 2\mbox{Tr} \left [-Z^\top S^{-1}\left( \frac{\partial}{\partial \vu_j}Z \right) Z^\top S^{-1}\vy\vy^\top S^{-1} \left( \frac{\partial}{\partial \vu_i}Z \right)\right.\\
&&\left.-Z^\top S^{-1}\vy\vy^\top S^{-1} \left( \frac{\partial}{\partial \vu_j}Z \right)Z^\top S^{-1} \left( \frac{\partial}{\partial \vu_i}Z \right)\right.\\
&&\left.+Z^\top S^{-1} \left( \frac{\partial}{\partial \vu_j}Z \right)Z^\top S^{-1} \left( \frac{\partial}{\partial \vu_i}Z \right)\right],
\end{eqnarray} 
where 
\begin{eqnarray}
 \frac{\partial^2}{\partial \vu_i \partial \vu_j}Z &=&X \left( \frac{\partial^2}{\partial \vu_i \partial \vu_j }C_{drd}^{\frac{1}{2}} \right) \Sigma^{\frac{1}{2}}.
\end{eqnarray} 
For DRD only, we can just derive the Hessian by replacing $C_{drd}^{\frac{1}{2}} \Sigma^{\frac{1}{2}}$ with $C_{drd}^{\frac{1}{2}}$.

\section{Proof of convexity of $\mathcal{L}_1(\vu)$}\label{conv}
We ignore the scaling $\frac{1}{2}$ here for simplicity, and write,
\begin{eqnarray}
\mathcal{L}_1(\vu) &=& \log |X \diag(e^\vu)X^\top+\nsevar I|\\
&=& \log | \diag(e^\vu)\frac{X^\top X}{\nsevar} + I|+const\\
&=& \log | \frac{X^\top X}{\nsevar} + \diag(e^{-\vu})|+ \log | \diag(e^{\vu})|+const.
\end{eqnarray} 
Let $V=\frac{X^\top X}{\nsevar}$, which is p.s.d., then
\begin{equation} 
\begin{aligned}
\frac{\partial }{\partial\vu} \log | V+ \diag(e^{-\vu})|&=-\diag\left((\diag(e^{-\vu})+ V)^{-1}\odot \diag(e^{-\vu})\right)\\
\frac{\partial^2}{\partial^2\vu} \log | V+ \diag(e^{-\vu})|&=\left(V(\diag(e^\vu)V+I)^{-1}\right)\odot\left(\diag(e^{-\vu})+V\right)^{-1},
\end{aligned}
\end{equation} 
where $\odot$ is the Hadamard product. Moreover, we know that
\begin{equation} 
\begin{aligned}
V(\diag(e^\vu)V+I)^{-1}&=\frac{X^\top X}{\nsevar}(\diag(e^\vu)\frac{X^\top X}{\nsevar}+I)^{-1}\\
&=X^\top(X\diag(e^\vu)X^\top +\nsevar I)^{-1}X\succeq 0.
\end{aligned}
\end{equation} 
Thanks to the Schur product theorem stating that the Hadamard product of two positive semi-definite matrices is also a positive semi-definite matrix, we have $\frac{\partial^2}{\partial^2\vu} \log | V+ \diag(e^{-\vu})|\succeq 0$, thus $\log | \frac{X^\top X}{\nsevar} + \diag(e^{-\vu})|$ is convex in $\vu$. In addition, $\log | \diag(e^{\vu})|$ is also convex in $\vu$. Therefore, $\mathcal{L}_1(\vu)$ is convex in $\vu$. 

%
%

\section{Proof of boundedness and non-emptiness of $\mathcal{F}(\vu)$}\label{bound}
We want to prove that for (eq.~\ref{mapA}), when $||\vu||\rightarrow\infty$, we have $\mathcal{F}(\vu) \rightarrow \infty$.
Rewrite $\mathcal{F}$ here,
\begin{eqnarray}
\mathcal{F}(\vu)={\vz^k}^\top\vh(\vu)-\mathcal{L}_\vh^*(\vz)+\mathcal{L}_1(\vu)+\mathcal{L}_3(\vu),
\end{eqnarray}
where $\vh(\vu)=e^{-\vu}$, $\mathcal{L}_1(\vu) =\frac{1}{2} \log |X \diag(e^\vu)X^\top+\nsevar I|$ and $\mathcal{L}_3(\vu) = \frac{1}{2}(\vu-b\vones)^\top K^{-1}(\vu-b\vones)$.

1) Each element in $\vh(\vu)$ is bounded by 0 and 1. Thus when $||\vu||\rightarrow\infty$, ${\vz^k}^\top\vh(\vu)$ will be bounded. \\
2) $K^{-1}$ is a positive semi-definite (p.s.d.) matrix. Thus $\mathcal{L}_3(\vu)$ is lower-bounded by 0. When $||\vu||\rightarrow\infty$, $\mathcal{L}_3(\vu)\ge 0$. The upper bound is unclear.\\
3) Denote $\Gamma=\diag(e^\vu)$ whose diagonal values are all nonnegative. Now we want to prove when $||\vu||\rightarrow\infty$, we have $\mathcal{L}_1(\vu) =\frac{1}{2} \log |X \Gamma X^\top+\nsevar I| \rightarrow \infty$. 

First, we note that $\mathcal{L}_1(\vu)$ has a lower bound. It can be easily shown that the eigenvalues of $X \Gamma X^\top+\nsevar I$ should be greater than or equal to $\nsevar$, given $X \Gamma X^\top$ is a p.s.d. matrix. 

If $||\vu||\rightarrow\infty$, we can assume $u_1\rightarrow\infty,...,u_s\rightarrow\infty$ where $\vu\in\mathbb{R}^p$ and $s\le p$, then $e^{u_i}\rightarrow \infty$, for all $i\in\{1,...,s\}$. For $\{u_i\}_{i=s+1}^p$, if $u_i$ is finite, $e^{u_i}$ will be finite; else if $u_i\rightarrow -\infty$, $e^{u_i}=0$. Thus $e^{u_i}$ is a finite value for all $i\in\{s+1,...,p\}$. We can write $\Gamma$ as an addition of two matrices $A$ and $B$, i.e. $\Gamma=A+B$.
\begin{eqnarray}
A=\begin{bmatrix}
e^{u_1}&&&&&&\\
&e^{u_2}&&&&0&\\
&&\ddots&&&&\\
&&&e^{u_s}&&&\\
&&&&0&&\\
&0&&&&\ddots&\\
&&&&&&0
\end{bmatrix}, 
\quad B=\begin{bmatrix}
0&&&&&&\\
&0&&&&0&\\
&&\ddots&&&&\\
&&&0&&&\\
&&&&e^{u_{s+1}}&&\\
&0&&&&\ddots&\\
&&&&&&e^{u_p}
\end{bmatrix}
\end{eqnarray}
The nonzero elements in $A$ are infinite values. The nonzero elements in $B$ are finite nonnegative values.
Let $M=X B X^\top+\nsevar I\in\mathbb{R}^{n\times n}$. $X B X^\top$ is a p.s.d. matrix. The smallest eigenvalue of $X B X^\top$ should be nonnegative. Therefore the smallest eigenvalue of $M$ is greater than or equal to $\nsevar$. This implies the invertibility of $M$. Since $M^{-1}$ is also positive definite, we can factorize $M^{-1}$ into $S\in\mathbb{R}^{n\times n}$ and $S^\top $, i.e. $M^{-1}=SS^\top $. Thus, we can write
\begin{eqnarray}
\mathcal{L}_1(\vu) &=&\frac{1}{2} \log |X A X^\top+X B X^\top+\nsevar I|\\
&=&\frac{1}{2} \log |X A X^\top+M|\\
&=&\frac{1}{2} \log |X A X^\top M^{-1}+I|+\frac{1}{2}\log|M|\\
&=&\frac{1}{2} \log |X A X^\top SS^\top +I|+\frac{1}{2}\log|M|\\
&=&\frac{1}{2} \log |S^\top X A X^\top S +I|+\frac{1}{2}\log|M|
\end{eqnarray}
Combining $X^\top S$ to be one matrix $Z\in\mathbb{R}^{p\times n}$, we can investigate the elements in $Z^\top A Z$. $Z^\top A Z$ should be equal to $\infty*\widetilde{Z}^\top \widetilde{Z}$ where $\widetilde{Z}\in\mathbb{R}^{s\times n}$ is a trimmed $Z$ by throwing away the rows with indices from $s+1$ to $p$. Therefore, nonzero eigenvalues of $\widetilde{Z}^\top \widetilde{Z}$ will turn into $\infty$ in $Z^\top A Z$. Zero eigenvalues will remain zero. 

Let $\lambda_i\ge 0$ denote the $i$th eigenvalue of $Z^\top A Z$, then 
\begin{eqnarray}
\mathcal{L}_1(\vu) &=&\frac{1}{2} \sum_{i=1}^p\log (\lambda_i+1)+\frac{1}{2}\log|M|
\end{eqnarray}
Since there exists at least one eigenvalue $\lambda_i$ in $Z^\top A Z$ approaching $\infty$, we can conclude that $\mathcal{L}_1(\vu)$ also approaches $\infty$ in such a case. 

Accordingly, if $||\vu||\rightarrow\infty$, we have $\mathcal{F}(\vu) \rightarrow \infty$, then there must exist a solution set for minimizing $\mathcal{F}(\vu)$. This validates the nonemptiness of the solution set. Furthermore, the solution set must be bounded. If it's not bounded, there must be a solution at $\infty$ with the minimal $\mathcal{F}(\vu)$, but this contradicts the assumption that $\mathcal{F}(\vu)\rightarrow \infty$ when $||\vu||\rightarrow\infty$. Therefore, we can claim that the solution set of $\mathcal{F}(\vu)$ is bounded and nonempty.

\end{document}